\documentclass[most, 11pt, a4paper]{delta-beta}
\usepackage{dblfloatfix}
\usepackage{ulem}
\usepackage{caption}
\usepackage{dramatist}
\usepackage{xspace}
\usepackage{pifont} %
\usepackage{multirow}
\usepackage{tcolorbox}
\usepackage{xltabular}
\usepackage{longtable}
\usepackage{hyperref}
\usepackage{amsfonts}
\usepackage{amsmath}
\usepackage{amssymb}
\usepackage{lineno}
\usepackage{multirow}
\usepackage{adjustbox}
\usepackage{float}
\usepackage{bm}

\usepackage[bottom]{footmisc}

\usepackage{CJKutf8}
\usepackage{subfigure}
\usepackage{setspace}
\usepackage{algorithm}
\usepackage{algpseudocode}  
\usepackage[dvipsnames]{xcolor}
\usepackage{dsfont}
\usepackage{array} %
\usepackage{tabularx} %
\usepackage{xcolor} %
\usepackage{tabularx}
\usepackage{booktabs}

\usepackage{lipsum}  %
\usepackage{multicol} %

\usepackage{natbib} 

\usepackage{algorithm}
\usepackage{algpseudocode}
\usepackage{enumitem}
\usepackage{amsthm}

\newtheorem{corollary}{Corollary}

\algrenewcommand\algorithmicrequire{\textbf{Input:}}
\algrenewcommand\algorithmicensure{\textbf{Output:}}

\usepackage{array}
\newcolumntype{L}[1]{>{\raggedright\arraybackslash}p{#1}}
\newcolumntype{C}[1]{>{\centering\arraybackslash}p{#1}}
\newcolumntype{R}[1]{>{\raggedleft\arraybackslash}p{#1}}

\def\obsarch{1,942}
\def\trainarch{170}

\makeatletter
\def\@BTrule[#1]{%
  \ifx\longtable\undefined
    \let\@BTswitch\@BTnormal
  \else\ifx\hline\LT@hline
    \nobreak
    \let\@BTswitch\@BLTrule
  \else
     \let\@BTswitch\@BTnormal
  \fi\fi
  \global\@thisrulewidth=#1\relax
  \ifnum\@thisruleclass=\tw@\vskip\@aboverulesep\else
  \ifnum\@lastruleclass=\z@\vskip\@aboverulesep\else
  \ifnum\@lastruleclass=\@ne\vskip\doublerulesep\fi\fi\fi
  \@BTswitch}
\makeatother

\addto\extrasenglish{

}

\addto\extrasUKenglish{

}

 {\begin{list}{}%
         {\setlength{\leftmargin}{#1}}%
         \item[]%
 }
 {\end{list}}

\reportnumber{001} %

\renewcommand{\today}{}

\usepackage{amsthm}
\newtheorem{theorem}{Theorem}[section]  

\usepackage{xcolor}
\usepackage{tcolorbox}

\newtcolorbox{scenariobox}[2][]{
  colback=#2!5!white,
  colframe=#2!75!black,
  fonttitle=\bfseries,
  title=#1,
  rounded corners,
  arc=3mm,
  boxrule=1.2pt,
  left=6pt,
  right=6pt,
  top=6pt,
  bottom=6pt,
  enhanced,
  drop shadow={opacity=0.3}
}

\newtcolorbox{keybox}[1][]{
  colback=yellow!10!white,
  colframe=orange!75!black,
  fonttitle=\bfseries\large,
  title={#1},  
  rounded corners,
  arc=2mm,
  boxrule=1.5pt,
  left=8pt,
  right=8pt,
  top=8pt,
  bottom=8pt,
  enhanced,
  drop shadow={opacity=0.4},
  borderline west={3pt}{0pt}{orange!75!black}
}

\title{Hardware Co-Design Scaling Laws via Roofline Modelling for On-Device LLMs}

\author[1,2,3,4]{Luoyang Sun}
\author[1,2,4]{Jiwen Jiang}
\author[4]{Yifeng Ding}
\author[4]{Fengfa Li}
\author[5]{Yan Song}
\author[2,3]{Haifeng Zhang}
\author[1]{Jian Ying}
\author[4,$\dagger$]{Lei Ren}
\author[4]{Kun Zhan}
\author[4,$\dagger$]{Wei Chen}
\author[4]{Yan Xie}
\author[6,$\ddagger$]{Cheng Deng}
\affil[1]{AI Lab, The Yangtze River Delta}
\affil[2]{Institution of Automation, Chinese Academy of Sciences}
\affil[3]{University of Chinese Academy of Sciences}
\affil[4]{Li Auto}
\affil[5]{University College London}
\affil[6]{The University of Edinburgh}

\renewcommand{\phi}{\varphi}

\renewcommand{\leq}{\leqslant}
\renewcommand{\geq}{\geqslant}

\renewcommand{\epsilon}{\varepsilon}
\renewcommand{\imath}{\mathrm{i}}

\newlength{\restsubwidth}
\newlength{\restsubheight}
\newlength{\restsubmoreheight}
\newcommand{\rest}[2]{%
        \settowidth{\restsubwidth}{\ensuremath{#2}}
        \settoheight{\restsubheight}{\ensuremath{{}_{#2}}}
        \ensuremath{{#1\hskip 0.5pt}_{\vrule\kern2pt\parbox[b][%
        4pt][b]{\the\restsubwidth}{%
                        \ensuremath{{}_{#2}}}}}
        }

\begin{abstract}

Vision-Language-Action Models (VLAs) have emerged as a key paradigm of Physical AI and are increasingly deployed in autonomous vehicles, robots, and smart spaces. In these resource-constrained on-device settings, selecting an appropriate large language model (LLM) backbone is a critical challenge: models must balance accuracy with strict inference latency and hardware efficiency constraints.
This makes hardware-software co-design a game-changing requirement for on-device LLM deployment, where each hardware platform demands a tailored architectural solution.
We propose a hardware co-design law that jointly captures model accuracy and inference performance. Specifically, we model training loss as an explicit function of architectural hyperparameters and characterise inference latency via roofline modelling. We empirically evaluate \obsarch{} candidate architectures on NVIDIA Jetson Orin, training \trainarch{} selected models for 10B tokens each to fit a scaling law relating architecture to training loss. By coupling this scaling law with latency modelling, we establish a direct accuracy-latency correspondence and identify the Pareto frontier for hardware co-designed LLMs. We further formulate architecture search as a joint optimisation over precision and performance, deriving feasible design regions under industrial hardware and application budgets.
Our approach reduces architecture selection from months to days. At the same latency as Qwen2.5-0.5B on the target hardware, our co-designed architecture achieves \textbf{19.42\%} lower perplexity on WikiText-2. To our knowledge, this is the first principled and operational framework for hardware co-design scaling laws in on-device LLM deployment. We will make the code and related checkpoints publicly available.

\end{abstract}

\begin{document}
\maketitle

\begin{figure}[h]
    \centering
    \includegraphics[width=0.95\linewidth]{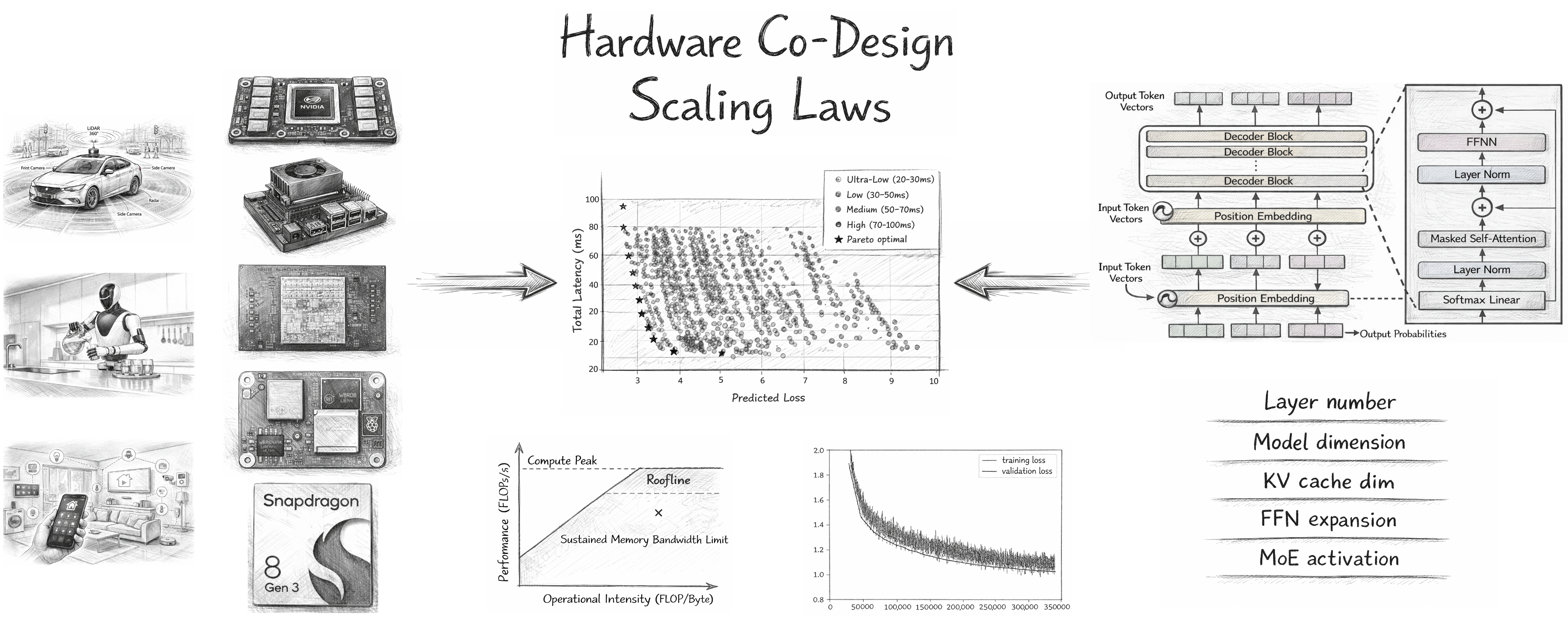}
    \caption{Hardware co-design scaling law for on-device LLMs. Architectural choices and hardware platforms jointly shape the loss-latency Pareto frontier, revealing Pareto-optimal configurations under system constraints.}
    \label{fig:overview}
\end{figure}

\renewcommand{\thefootnote}{\fnsymbol{footnote}}
    \footnotetext[1]{Version: v1 (major update on Feb 10, 2026).}
    \footnotetext[2]{Project Leader.}
    \footnotetext[3]{Correspondence to \href{mailto:davendw49@gmail.com}{Cheng Deng}.}
    
\renewcommand{\thefootnote}{\arabic{footnote}}

\newpage
\setcounter{tocdepth}{2}
\tableofcontents
\newpage

\section{Introduction}

Large language models are increasingly deployed in embodied AI systems such as autonomous vehicles and mobile robots, where they serve as high-level planners within Vision–Language-Action frameworks~\cite{kim2024openvla,zitkovich2023rt,sapkota2025vision}. However, on-device platforms face strict constraints on memory, bandwidth, power, and latency that fundamentally reshape model design~\cite{zhou2025hierarchical,yang2023llm4drive}. Architectures optimized for cloud GPUs often become infeasible on the edge: high-accuracy models may violate latency budgets, while latency-optimized pipelines can degrade accuracy. This tension motivates hardware–software co-design, where architectural choices are explicitly guided by hardware capabilities and deployment constraints~\cite{guo2025survey}.

The fundamental challenge stems from the irregular compute–memory profile of transformers~\cite{mobilellm,deng2025plm}: attention is bandwidth-bound, feedforward layers are compute-bound, and KV-cache stresses on-chip memory. Despite high theoretical throughput from AI-SoCs, LLM inference rarely reaches peak utilization. Instead, performance is dictated by arithmetic intensity, on-chip locality, and workload patterns like KV-cache footprint and MoE routing. Architectural modification can shift operations across regimes in the hardware Roofline model~\cite{williams2008roofline,llmviewer}, shown in~\autoref{fig:roofline_illustrate}. A representative example is scaling model depth and width in Transformer-based LLMs. Increasing depth results in linear growth in both computation and memory traffic, as each additional layer introduces a fixed amount of parameter reads and arithmetic operations. In contrast, increasing model width leads to quadratic growth in parameter size and memory I/O, since both attention and feed-forward layers scale with the square of the hidden dimension. Under batch-1 inference on edge devices, where weight reuse across tokens is limited and on-chip cache capacity is insufficient to hold model parameters, inference latency is dominated by repeated weight loads from off-chip memory. As a consequence, arithmetic intensity remains low and does not scale proportionally with model width, pushing execution into a memory bandwidth–limited regime in the roofline model~\cite{bian2025scaling1,bian2025scaling2}. This mismatch between scaling behavior and hardware characteristics motivates hardware-aware architectural design beyond naive depth and width scaling.

\begin{figure}[htbp]
    \centering
    \subfigure[Roofline Model]{
        \centering
        \includegraphics[height=3.2cm]{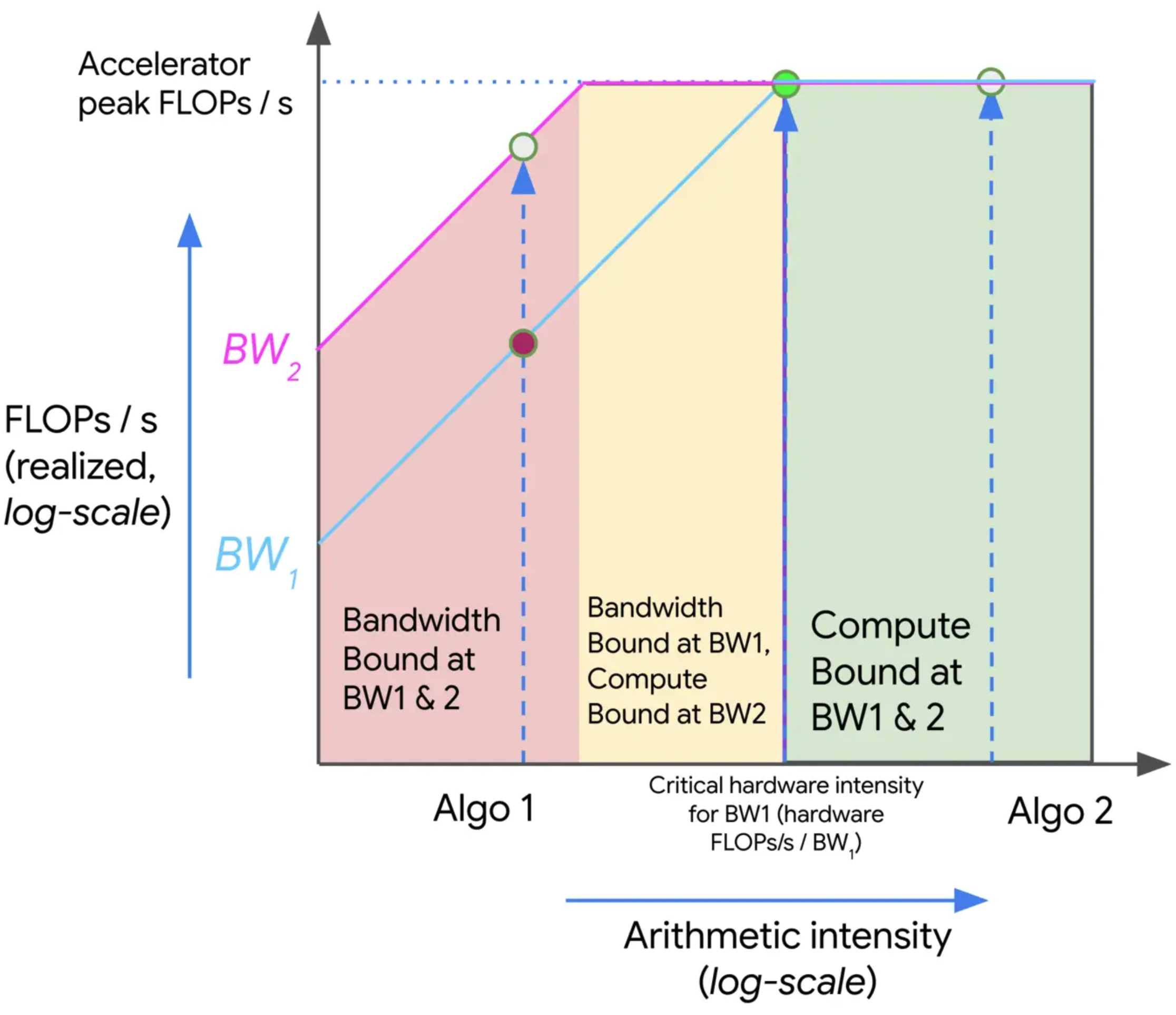}
        \label{fig:roofline_illustrate}
    }
    \subfigure[Neural Architecture Search]{
        \centering
        \includegraphics[height=3.2cm]{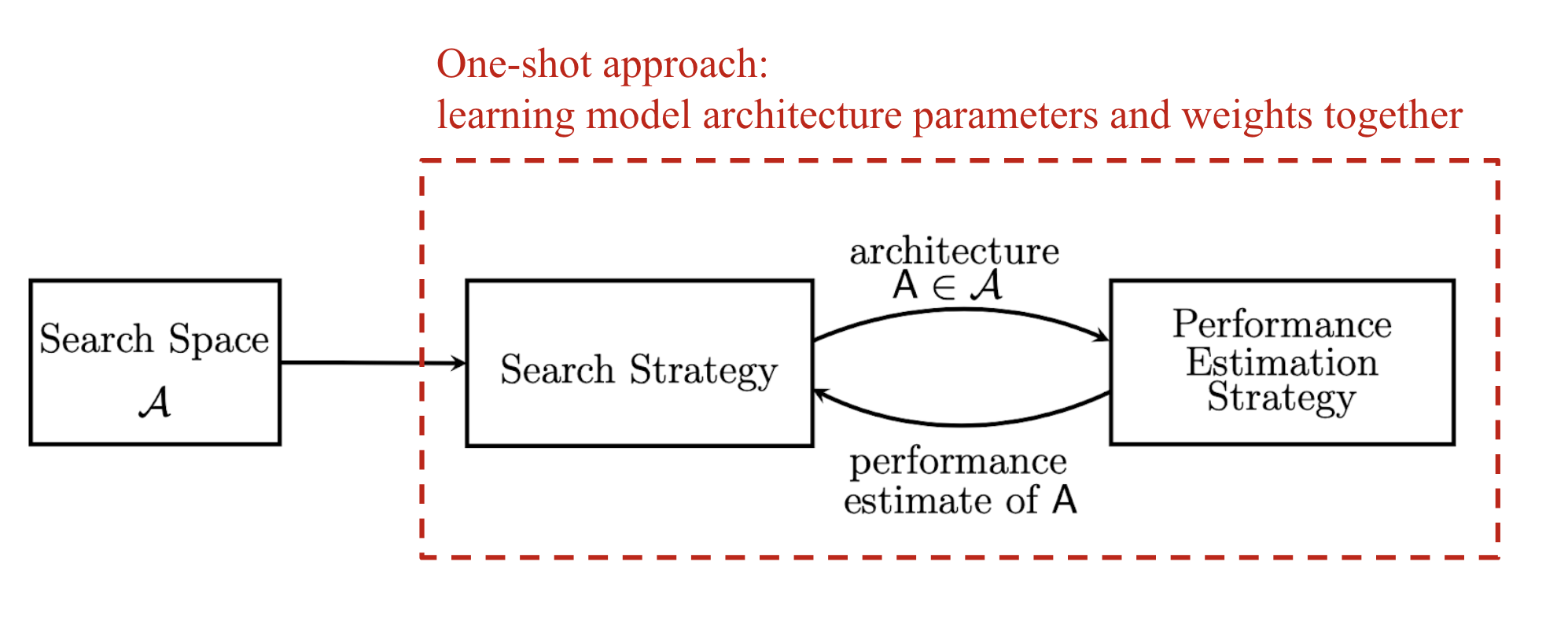}
        \label{fig:nas_illustrate}
    }
    \subfigure[Pareto Frontier]{
        \centering
        \includegraphics[height=3.2cm]{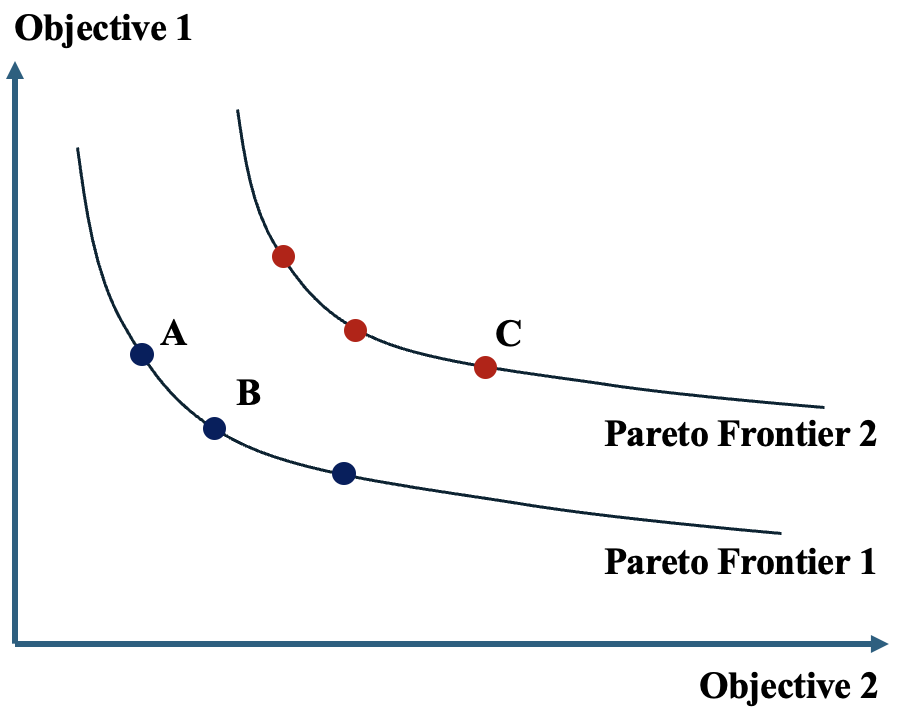}
        \label{fig:pareto_illustration}
    }
    \caption{Hardware co-design strategy preliminaries. (a) Roofline model~\cite{williams2008roofline} comparing achieved performance against theoretical hardware limits. (b) One-shot NAS~\cite{elsken2019neural} framework jointly optimizing architecture and weights. (c) Pareto frontier visualizing accuracy-efficiency trade-offs~\cite{cheng2018searching}.}
    \label{fig:intro_pre}
\end{figure}

In order to find an optimal model, Neural architecture search (NAS)~\cite{zoph2016neural} usually is an appropriate way, shown in~\autoref{fig:nas_illustrate}. NAS has traditionally focused on optimizing a single objective, such as validation loss or accuracy, often under loosely defined computational budgets~\cite{elsken2019neural}. While effective in unconstrained settings, these approaches are ill-suited to edge and on-device platforms, where latency–accuracy trade-offs are unavoidable. As shown in~\autoref{tab:llm_tradeoff}, LLM inference system design involves several fundamental trade-offs~\cite{cheng2025lmcache,yang2025kvlink,kurtic2025give,egiazarian2025bridging,yin2025specpipeacceleratingpipelineparallelismbased,xu2025characterizing,deng2025plm}. This paper focuses on the interplay between model loss and inference latency in hardware co-designed LLMs (bold in~\autoref{tab:llm_tradeoff}), which is commonly observed in practice but has not been systematically characterized. This trade-off motivates a Pareto-based approach to architecture search~\cite{cheng2018searching}, visualized as~\autoref{fig:pareto_illustration}. Instead of a single optimal model, Pareto optimization identifies a frontier of non-dominated architectures balancing accuracy and latency. This allows designers to directly select models fitting specific deployment constraints, avoiding exhaustive enumeration~\cite{mobilellm,tang2024rethinking}.

\begin{table}[h]
\centering
\caption{Trade-off Analysis for LLM Inference System Design}
\label{tab:llm_tradeoff}
\small
\resizebox{\linewidth}{!}{%
\begin{tabular}{p{0.30\linewidth} p{0.30\linewidth} p{0.33\linewidth}}
\toprule
\textbf{Scenario} & \textbf{Trade-off Targets} & \textbf{Pareto Trade-off Factors} \\
\midrule
Memory allocation optimization 
& Throughput vs. Single-Query Latency 
& Batch size selection and KV-cache reuse strategies \\ \midrule

Quantization strategy selection 
& Memory Cost vs. Model Precision 
& INT8 / INT4 quantization levels and accuracy verification methods \\ \midrule

Distributed inference system 
& Network Communication Cost vs. Compute Parallelism 
& Model parallelism granularity and pipeline depth \\ \midrule

\textbf{Hardware co-designed LLM}
& \textbf{Model Loss vs. Inference Latency}
& \textbf{Equivalent parameter count and inference time over hardwares} \\
\bottomrule
\end{tabular}
}
\end{table}

We propose a hardware-aware modeling framework (shown in~\autoref{fig:overview}) for on-device LLMs that jointly captures accuracy and inference performance. We model loss via architectural hyperparameters and predict latency using roofline analysis, providing a unified view of the accuracy–latency trade-off. We validate the framework on NVIDIA Jetson Orin by benchmarking thousands of candidate architectures and training a subset to fit a parameter–loss scaling law. Combining this law with hardware-level latency modeling yields the Pareto frontier for the target system. We extend this into a theoretical joint optimization formulation over precision and performance, deriving feasible design regions under practical constraints. 

Our main contribution can be listed as follows:

\begin{itemize}[leftmargin=1.2em]
    \item We develop a hardware co-design law that combines loss scaling laws with roofline-based latency modeling, enabling an explicit Pareto characterization of accuracy–latency trade-offs under fixed hardware constraints for on-device LLMs. To the best of our knowledge, this is the first practical and operational hardware co-design scaling law for on-device LLMs.

    \item We benchmark approximately \obsarch{} LLM architectures on NVIDIA Jetson Orin to identify hardware-aligned architectural patterns, select \trainarch{} representative models, and train each for 10B tokens to empirically fit the loss scaling law and validate the resulting accuracy–latency Pareto structure.

    \item We extend the empirical framework into a principled theoretical formulation that casts architecture search as a joint optimization problem over precision and performance, deriving feasible design regions under industrial hardware and application budgets.

    \item Since this approach reduces architecture selection time from months to only few days, we will release the full methodology, codebase, trained models, and detailed evaluation protocols to support reproducibility and support the development of hardware co-design community. 
\end{itemize}

The paper is organized as follows. \autoref{sec:related_work} reviews related works. In~\autoref{sec:formulation}, we formulate the hardware co-design law under on-device constraints. In~\autoref{sec:pareto_nas}, we present Pareto-optimal architecture discovery via roofline modeling, detailing how optimal model architectures are identified for a given hardware platform. As for~\autoref{sec:theory}, we extend the empirical framework to a theoretical formulation, casting architecture search as a joint optimization problem over precision and performance. Finally, we discuss practical usage and application scenarios of the proposed hardware co-design law. 

\section{Related Work}
\label{sec:related_work}

\subsection{Efficient LLM Architectures}

Recent LLM designs prioritize inference efficiency beyond parameter scaling. Key directions include: (1) \textit{sparse activation} via MoE routing~\cite{guo2025deepseek,agarwal2025gpt}, (2) \textit{KV-cache reduction} through multi-head latent attention (MLA)~\cite{deng2025plm} or sliding-window attention~\cite{team2025gemma}, (3) \textit{sub-quadratic attention} using linear-complexity mechanisms~\cite{team2025kimi,team2025minimax}, and (4) \textit{hybrid architectures} combining SSM with attention~\cite{blakeman2025nvidia}. Gated attention variants~\cite{qiu2025gated,sun2025gta} further enable token-dependent compute modulation. These innovations fundamentally shift inference bottlenecks from compute-bound to memory- or routing-sensitive regimes.

\subsection{On-Device LLM Deployment}

Deploying LLMs on edge devices requires co-optimization across model, compression, and system layers. SLMs such as Qwen~\cite{yang2024qwen2_5}, MiniCPM~\cite{hu2024minicpm}, and SmolLM~\cite{bakouch2025smollm3} achieve strong performance through data-efficient training. Quantization methods (AWQ~\cite{lin2024awq}, GPTQ~\cite{frantar2022gptq}) and inference engines (vLLM~\cite{kwon2023efficient}, MLC-LLM~\cite{mlcllm}, PowerInfer~\cite{song2024powerinfer}) enable efficient execution on heterogeneous hardware. Architecturally, deeper designs with shared embeddings~\cite{mobilellm} and sparsity-aware mechanisms~\cite{deng2025plm} demonstrate favorable accuracy–efficiency trade-offs for resource-constrained deployment.

\subsection{Hardware-Aware Architecture Optimization}

Neural Architecture Search (NAS)~\cite{zoph2018learning,liu2019darts} automates architecture design but faces challenges in LLM contexts: prohibitive search costs, difficulty incorporating latency/memory constraints, and limited interpretability. Profiling tools like LLM-Viewer~\cite{llmviewer} and LLMCompass~\cite{llmcompass} provide hardware-aware analysis but lack integration with architecture search. Studies on depth–width trade-offs~\cite{tay2021scale,bian2025scaling1} yield inconsistent conclusions, often neglecting deployment constraints.

\vspace{-0.2cm}
\paragraph{Positioning.}
This work bridges LLM architecture design with hardware-aware performance modeling. We extend roofline analysis to explicitly model how architectural choices, MoE sparsity, attention variants, KV-cache strategies, map to compute- or bandwidth-limited regimes on edge devices, enabling architecture search guided by hardware co-design law rather than post-hoc benchmarking.
\section{Formulating Hardware Co-Design Law for on-Device LLM}
\label{sec:formulation}

Classical scaling laws~\citep{kaplan2020scaling, hoffmann2022chinchilla} characterize the relationship between model size, data size, and training compute under fixed training budgets. In contrast, this work focuses on the deployment regime, where the objective is to identify the optimal architecture $\bm{\theta}^*$ under a fixed inference latency budget $T_{\text{lat}}$ and precision constraints. In this section, we formalize this problem and introduce a hardware co-design formulation for on-device LLMs.

\subsection{Implicit Optimization Objective of Hardware Co-Design Law}
\textbf{Optimization Objective.} We seek to minimize validation loss subject to latency and memory constraints:
\begin{equation}
\min_{\bm{\theta} \in \Theta} \; \mathcal{L}(\bm{\theta}) \quad \text{s.t.} \quad T(\bm{\theta}; H, W) \leq T_{\text{lat}}, \quad M(\bm{\theta}; W) \leq M_{\text{budget}}
\label{eq:problem}
\end{equation}
where $\bm{\theta} = (l, d, d_m, r, \rho)$ denotes the model architecture. Specifically, $l$ is the number of transformer layers (depth), $d$ is the model width (hidden dimension), and $d_m = d_h \times n_{kv}$ is the key–value cache dimension, determined by the per-head dimension $d_h$ and the number of key–value heads $n_{kv}$ in grouped-query attention (GQA), which directly governs KV-cache memory footprint and bandwidth consumption during autoregressive decoding. The FFN expansion ratio $r$ controls the size of the intermediate feed-forward layer, with intermediate dimension $r \cdot d$, and therefore dominates per-token compute. For mixture-of-experts (MoE) models, $\rho = K/E \in (0,1]$ denotes the expert activation rate, where $K$ experts are selected from a pool of $E$ experts per token. In this setting, $r$ represents the total expansion ratio across all activated experts: if each expert has per-expert expansion $r_{\text{single}}$, then $r = K \cdot r_{\text{single}}$, ensuring that FFN compute remains comparable across different sparsity levels.

The latency surrogate $T(\bm{\theta}; H, W)$ models the end-to-end latency under context encoding and autoregressive generation, depending on both architectural and hardware characteristics. Hardware parameters $H = (\pi_H, \beta_H)$ include the peak compute throughput $\pi_H$ (FLOPS) and sustained memory bandwidth $\beta_H$ (Bytes/s). Under roofline analysis, $\pi_H$ and $\beta_H$ jointly determine whether inference is compute- or bandwidth-bound. The workload configuration $W = (B, S_{\text{in}}, S_{\text{out}})$ specifies the batch size $B$, input sequence length $S_{\text{in}}$, and output sequence length $S_{\text{out}}$, which collectively affect attention complexity and KV-cache access patterns during decoding. The memory surrogate $M(\bm{\theta}; W)$ estimates memory consumption based on model architecture and workload configuration.

\textbf{Problem Tractability.}  
Due to the high dimensionality of architectural hyperparameters in modern LLMs, the search space induced by~\autoref{eq:problem} is prohibitively large to enumerate. As a result, directly solving this constrained optimization problem is computationally infeasible in realistic deployment settings.
Rather than performing brute-force search, we approximate the optimization landscape through explicit surrogate models that capture stable trends in both learning dynamics and system behavior. Specifically, in the following two subsections, we model validation loss using a parametric polynomial approximation fitted from empirical training runs, and characterize inference latency via a roofline-based hardware performance model. This approach enables principled and scalable identification of Pareto-optimal architectures under fixed hardware constraints.

\subsection{Precision Modeling via Loss}
\label{sec:loss_model}
Beyond reducing the cost of hyperparameter search, our goal is to capture systematic and generalizable relationships between architectural design choices and model quality. We therefore construct an explicit analytical surrogate for validation loss, approximating the true objective in~\autoref{eq:problem} using empirical scaling behavior observed during training.

As described in~\autoref{sec:pareto_nas}, we train \trainarch{} architectures covering both dense and MoE models and fit an empirical scaling law based on~\autoref{eq:loss}. This polynomial approximation enables direct prediction of validation loss from architectural parameters, facilitating efficient exploration of the design space.

Consistent with prior unified scaling analyses~\cite{clark2022unified,krajewski2024scaling}, we model loss as a separable function of architecture components. 
Our empirical results reveal that sparsity-driven and base-capacity terms follow different width scaling exponents,
\begin{equation}
\hat{\mathcal{L}}(\bm{\theta}) = \frac{\kappa_l}{l^{\alpha_l}} + \frac{\kappa_\rho \cdot \rho^{\alpha_\rho}}{r^{\alpha_r} d^{\beta_1}} + \frac{\kappa_d}{r^{\alpha_r} d^{\beta_2}} + \frac{\kappa_m}{d_m^{\alpha_m}} + \mathcal{L}_\infty,
\label{eq:loss}
\end{equation}
where $\phi = \{\kappa_l, \kappa_\rho, \kappa_d, \kappa_m, \alpha_l, \alpha_\rho, \alpha_r, \alpha_m, \beta_1, \beta_2, \mathcal{L}_\infty\}$ are fitted parameters. Note that the KV-cache dimension satisfies $d_m = d / \mathrm{gqa}$, where $\mathrm{gqa} = n_h / n_{kv}$ is the GQA group ratio; this reparameterization is used in the theoretical analysis of \autoref{sec:theory}.

\subsection{Performance Modeling via Latency}
\label{sec:latency_model}
For the latency term, we derive an approximate mathematical expression grounded in roofline analysis, as latency is determined by the interaction between model compute, memory access patterns, and hardware characteristics. The hardware parameters required for this modeling include peak compute throughput and sustained memory bandwidth. While our current formulation targets conventional memory-centric accelerator architectures, it naturally generalizes to other hardware systems, which we leave for future work.

In this paper, we derive a first-principles latency model grounded in the roofline framework~\citep{williams2009roofline}. The roofline model characterizes computational kernels by arithmetic intensity $\mathcal{I} = \mathcal{F}/\mathcal{M}$ (FLOPs per byte). Given hardware with peak compute $\pi_H$ (FLOPS) and memory bandwidth $\beta_H$ (Bytes/s), ideal latency under the theoretical roofline model $\mathcal{T}$ satisfies:
\begin{equation}
\mathcal{T} = \max\left(\frac{\mathcal{F}}{\pi_H}, \, \frac{\mathcal{M}}{\beta_H}\right).
\label{eq:roofline}
\end{equation}

Here we directly give out the total inference latency under the roofline modeling. For a model with $l$ layers, $S_{\text{in}}$ input tokens, and $S_{\text{out}}$ generated tokens, the end-to-end latency $T_{\text{total}}$, including prefill and decode, is formulated as follows.
\begin{equation}
\hat{T_{\bm{\theta}}} = T_{\text{total}}(S_{\text{in}}, S_{\text{out}})
= l \cdot T_{\text{layer}}^{\text{pre}}(S_{\text{in}})
+ \sum_{S=1}^{S_{\text{out}}} l \cdot T_{\text{layer}}^{\text{dec}}(S+S_{in}),
\end{equation}

All FLOP counts and memory-traffic analysis are derived in~\autoref{app:problem_formulation}.

\subsection{Pareto-optimal architectures}

During the empirical experiments, we search for optimal LLM architectures via Pareto frontier analysis. 
An architecture $\bm{\theta}^\star$ is said to be Pareto-optimal if
\begin{equation}
\nexists\, \bm{\theta} \neq \bm{\theta}^\star \;\quad\text{s.t.}\; \quad
\mathcal{L}(\bm{\theta}) \le \mathcal{L}(\bm{\theta}^\star) \;\land\; 
T(\bm{\theta};H,W) \le T(\bm{\theta}^\star;H,W),
\end{equation}
with at least one inequality strict. The set of all such $\bm{\theta}^\star$ defines the Pareto frontier in the loss--latency plane.
In the following section, we present empirical experiments for Pareto-optimal architecture discovery under fixed hardware and context constraints.

By combining the empirical loss model (\autoref{sec:loss_model}) with the analytical latency model (\autoref{sec:latency_model}), we obtain two jointly optimizable objective functions that capture accuracy and inference efficiency, respectively. Within the Pareto optimality framework, we perform both empirical evaluation and analytical reasoning to identify architectures that optimally trade off validation loss and end-to-end inference latency on a given hardware platform.

We refer to the resulting architecture selection principle, together with its associated parameter scaling behavior under hardware constraints, as the hardware co-design scaling law. This law serves as a practical guideline for selecting and deploying on-device large language models under strict latency and resource budgets.

\def\nas{\textsc{PLAS}}
\section{Pareto-Optimal Architecture Search}
\label{sec:pareto_nas}

\begin{figure}[h]
    \centering
    \includegraphics[width=0.9\linewidth]{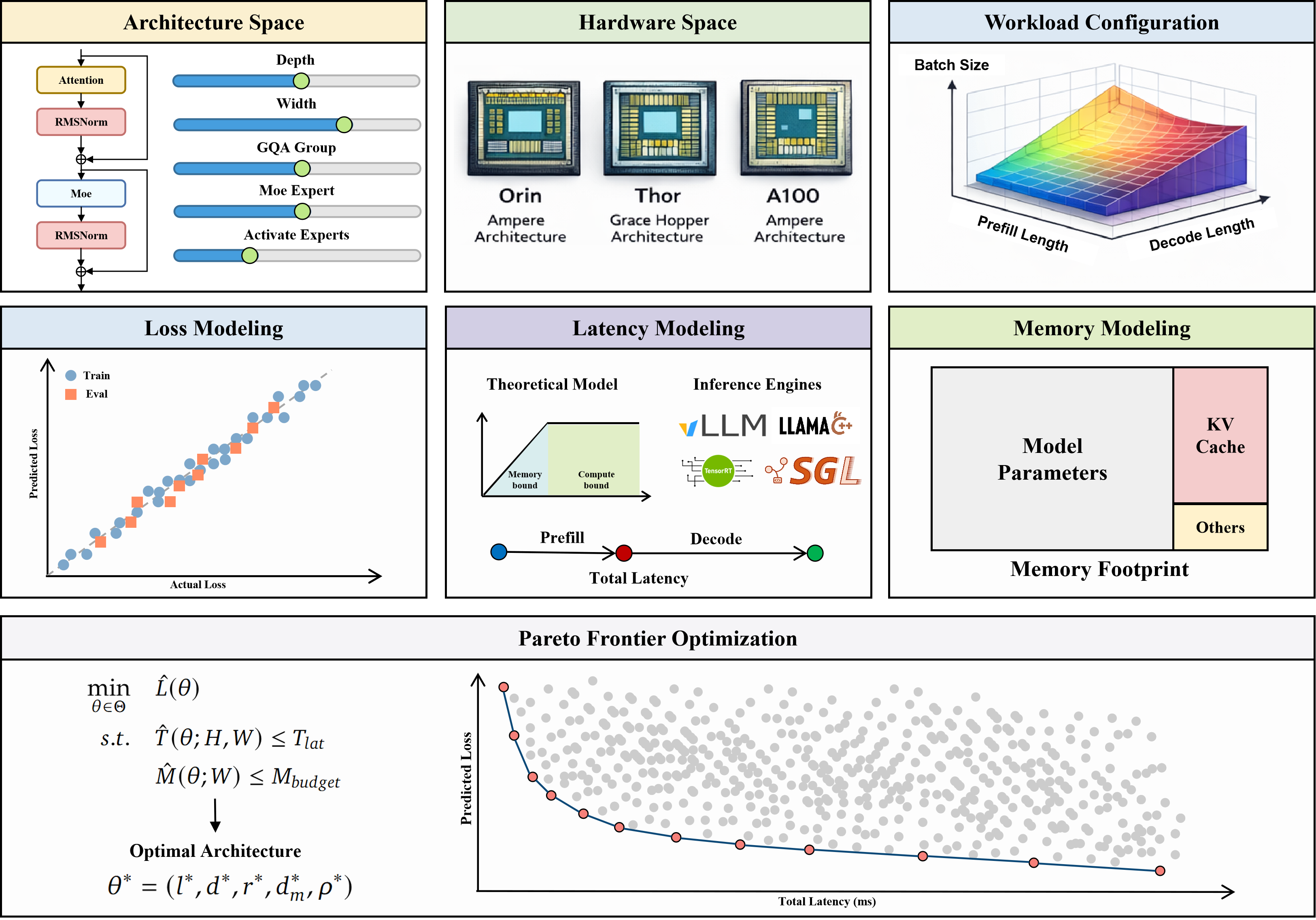}
    \caption{Overview of Pareto-optimal LLM Architecture Search framework (\nas{}). The framework integrates (1) empirical loss modeling via scaling law fitting, (2) roofline-based latency estimation, and (3) Pareto frontier construction to enable hardware-aware architecture selection.}
    \label{fig:overview_plas}
\end{figure}

This section presents \nas{} (\textbf{P}areto-optimal \textbf{L}LM \textbf{A}rchitecture \textbf{S}earch), a framework that jointly models training loss and inference latency to enable hardware-aware architecture selection. We first construct an empirical loss model by fitting results from \trainarch{} trained architectures to approximate validation loss without exhaustive search. We then characterize inference latency through roofline-based analytical modeling and practical measurements on edge platforms. Finally, we integrate both models to derive Pareto frontiers and demonstrate how they guide architecture selection under different application-specific latency budgets. \autoref{fig:overview_plas} illustrates the overall workflow.

\subsection{Loss Prediction via Scaling Laws}
\label{sec:loss_prediction}

Obtaining a high-fidelity parametric scaling law is non-trivial. Our fitting is grounded in \trainarch{} trained Transformer configurations spanning both sparse (MoE) and dense architectures, each trained for a fixed 10B-token budget under tightly controlled settings. The architectural configurations are carefully selected to span the full design space, jointly varying depth, width, MoE sparsity, FFN expansion ratio, and KV-cache dimensions (detailed search space in~\autoref{app:search_space}), while avoiding degenerate or ill-conditioned regimes.

\subsubsection{Pre-training Protocol}
\label{sec:pretrain_protocol}

All models share the following training setup to ensure fair comparison:

\begin{itemize}[leftmargin=*, itemsep=2pt]
    \item \textbf{Training Data.} Each configuration is trained on 10B tokens comprising a mixture of general corpus, mathematical reasoning, and code data, sufficient to observe scaling behavior while remaining computationally tractable. The training corpus will be released upon publication.

    \item \textbf{Optimization.} All models are trained using the AdamW optimizer with $\beta_1 = 0.90$, $\beta_2 = 0.95$, and weight decay of $0.01$. The learning rate follows a cosine decay schedule from $1\times10^{-4}$ to $1\times10^{-6}$, with linear warmup over the first $0.2\%$ of training steps. QK-Norm is employed to enhance training stability, particularly for MoE configurations. All experiments use a global batch size of 256.

    \item \textbf{Evaluation.} Model performance is evaluated using upstream validation loss on a held-out subset of approximately 1B tokens, averaged over the final 10 optimization steps to mitigate variance. We further assess generalization by reporting perplexity on the WikiText-2 test set.
\end{itemize}

Further details on the pre-training protocol are provided in~\autoref{app:pretrain_details}.

\subsubsection{Scaling Law Fitting}
\label{sec:scaling_fit}

\begin{figure}[h]
\centering
\includegraphics[width=0.5\linewidth]{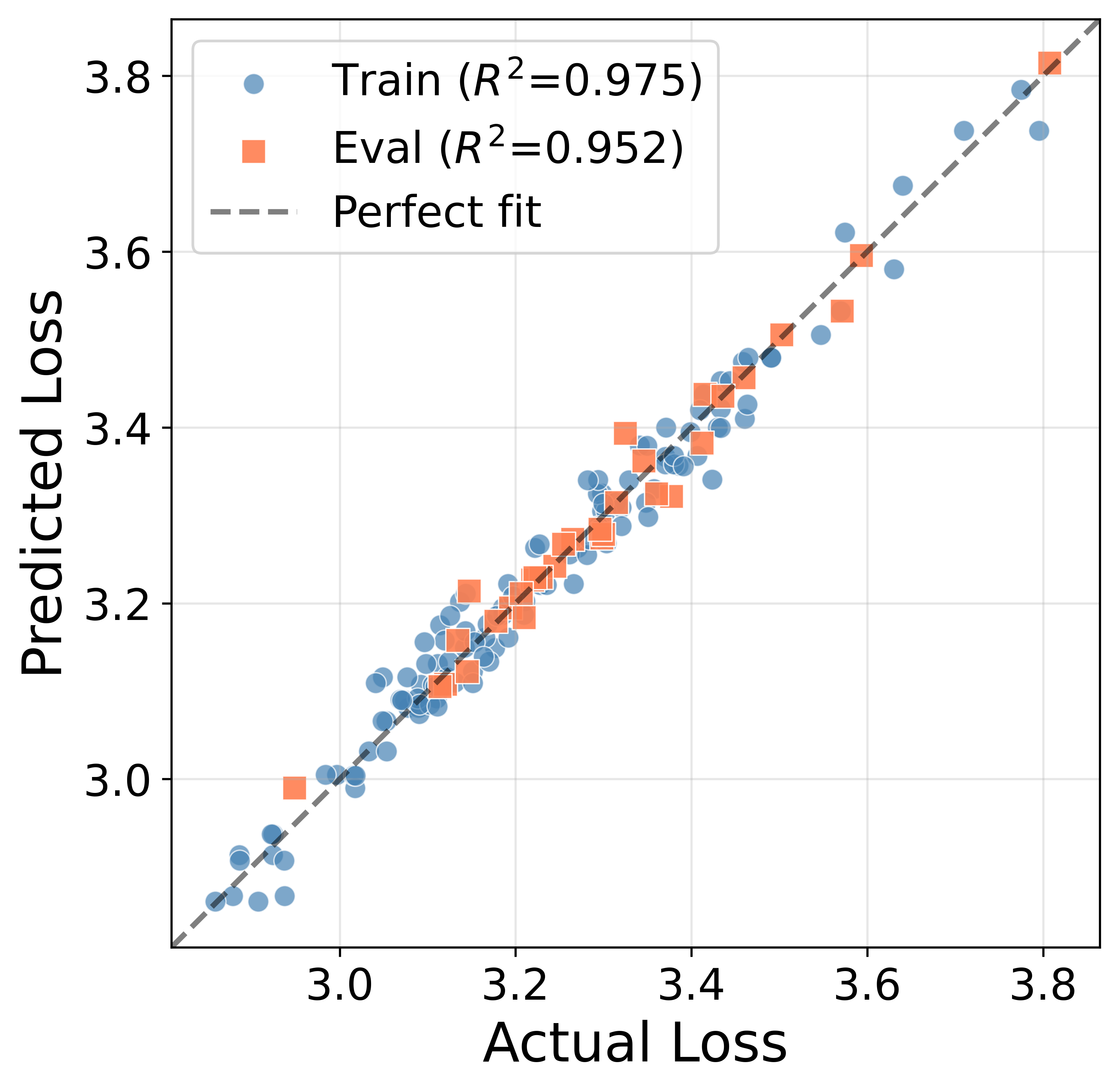}
\caption{Scaling law fit quality. Training $R^2 = 0.975$ (138 configurations); validation $R^2 = 0.952$ (32 held-out configurations).}
\label{fig:loss_fit}
\end{figure}

We fit a parametric scaling law of the form in~\autoref{eq:loss} using nonlinear least squares on 120 training configurations, with 17 held-out configurations reserved for validation. This extensive and structured exploration enables a stable fit with strong generalization: as shown in~\autoref{fig:loss_fit}, the resulting model achieves a training $R^2$ of \textbf{0.975} and a validation $R^2$ of \textbf{0.952}. Such predictive accuracy is difficult to obtain in practice, as loss landscapes across architectural dimensions are highly non-convex and often confounded by parameter coupling effects.

Despite operating over a substantially more heterogeneous architectural space that includes both dense and sparse models, the fitted scaling law exhibits stable and consistent exponents across depth, width, sparsity, and FFN expansion. This level of consistency is comparable to prior empirical scaling analyses~\cite{abnar2025parameters,krajewski2024scaling}, while achieving stronger generalization performance on held-out configurations, suggesting improved robustness beyond architecture-specific fitting. The fitted coefficients and complete functional form are provided in~\autoref{app:scaling_coefficients}.

Crucially, the quality of the fit validates our central premise: architecture-level loss can be modeled explicitly and predictably when training compute, data budget, and optimization protocol are fixed, thereby enabling principled extrapolation and Pareto-optimal architecture selection under hardware constraints.

\subsection{Latency Modeling}
\label{sec:latency_modeling}

To enable efficient architecture search, we require fast and accurate latency estimation that can evaluate tens of thousands of configurations without exhaustive measurement. Our framework employs roofline-based analytical modeling as the primary evaluation backend, with empirical validation for top candidates.

\subsubsection{Roofline-Based Prediction}
\label{sec:roofline_prediction}

We estimate inference latency by classifying each operator as compute-bound or memory-bound based on its arithmetic intensity relative to hardware capabilities. For each operator, latency is estimated from FLOPs, memory access volume, and hardware peak throughput (compute capacity and memory bandwidth). This analytical approach enables evaluation of \textbf{50,000+ configurations in approximately 20 minutes}, making it ideal for large-scale exploration.

To ensure prediction fidelity, we validate top Pareto candidates via empirical measurement using the vLLM inference engine with subprocess isolation for accurate GPU memory accounting. The roofline predictions exhibit strong correlation with measured latencies (details in~\autoref{app:latency_details}), confirming the reliability of our analytical approach for architecture ranking.

\subsubsection{Workload Configuration}
\label{sec:workload_config}

For on-device deployment targeting VLA workloads in autonomous driving, we focus on batch size $B=1$ with 1,024 input tokens and 16 output tokens. Under these settings:

\begin{itemize}[leftmargin=*, itemsep=2pt]
    \item \textbf{Prefill Latency} scales with input sequence length due to attention computation ($O(S^2)$ complexity), and is primarily compute-bound at moderate sequence lengths.
    
    \item \textbf{Decode Latency} is dominated by weight loading from memory, as each token generation requires accessing the full model weights while performing minimal computation per byte loaded.
\end{itemize}

The appropriate optimization target depends on workload characteristics: decode latency for interactive or streaming applications where per-token throughput is critical, prefill latency for long-context processing with short outputs, and total end-to-end latency for balanced tasks. As we demonstrate in \autoref{sec:pareto}, different optimization targets yield markedly different optimal architectures, motivating our multi-objective Pareto analysis. Further details on latency scaling behavior across sequence lengths and batch sizes are provided in~\autoref{app:latency_details}.

\subsection{Pareto Frontier Analysis}
\label{sec:pareto}

Given explicit loss and latency models, we cast architecture selection as a bi-objective optimization problem and identify Pareto-optimal designs that jointly minimize validation loss and inference latency. This formulation enables systematic exploration of the accuracy–efficiency trade-off and supports principled, scenario-aware architecture selection under hardware constraints.

\subsubsection{Frontier Construction}
\label{sec:frontier_construction}

Given loss predictions $\{\hat{\mathcal{L}}(\bm{\theta}_i)\}$ and latency estimates $\{\hat{T}(\bm{\theta}_i)\}$ for a set of architectural configurations, we identify the practical Pareto frontier as:
\begin{equation}
\label{prac_pareto}
\mathcal{P} = \{\bm{\theta}_i : \nexists\, \bm{\theta}_j \; \text{s.t.} \; \hat{\mathcal{L}}(\bm{\theta}_j) < \hat{\mathcal{L}}(\bm{\theta}_i) \land \hat{T}(\bm{\theta}_j) < \hat{T}(\bm{\theta}_i)\}
\end{equation}

We construct the Pareto frontier using an adaptive search strategy. Starting from an initial set of architectures generated via Latin hypercube sampling to ensure broad coverage of the design space, we identify the current Pareto-optimal set based on predicted loss and latency. We then iteratively refine the search by sampling new configurations in sparsely covered regions of the frontier and in local neighborhoods of Pareto-optimal points. This process repeats until the frontier stabilizes and no further improvements are observed.

\subsubsection{Precision–Performance Trade-off}
\label{sec:precision_tradeoff}

\autoref{fig:pareto_frontier} presents Pareto frontiers under three latency objectives (prefill, decode, and total), each comparing FP16 and INT8 precision. Across all scenarios, INT8 quantization consistently shifts the frontier toward lower latency at equivalent loss, demonstrating clear efficiency gains from reduced-precision inference. 

\begin{figure}[h]
    \centering
    \includegraphics[width=0.9\linewidth]{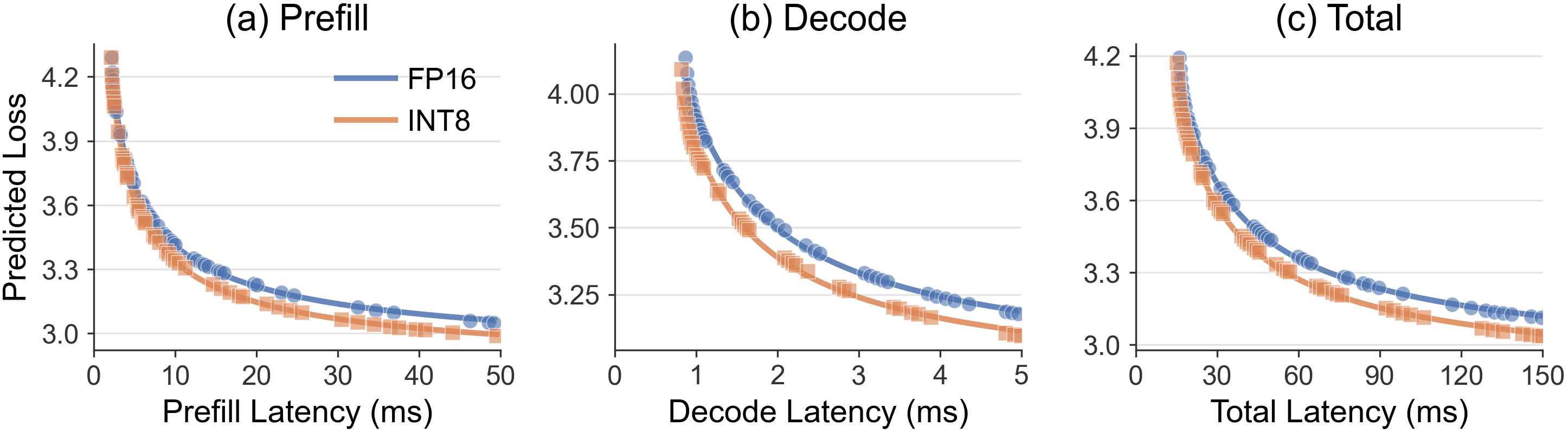}
    \caption{Pareto frontiers under prefill (1,024 tokens), decode (16 tokens), and total latency optimization on NVIDIA Jetson Orin, comparing FP16 and INT8 precision.}
    \label{fig:pareto_frontier}
\end{figure}

However, the observed speedup is notably less than the theoretical $2\times$ improvement. This sub-linear scaling arises from two primary factors: (1) INT8 acceleration applies only to linear operations (matrix multiplications), while non-linear components—attention softmax, layer normalization, and activation functions—remain in higher precision; and (2) quantization and dequantization overhead at layer boundaries partially offsets the computational savings from reduced-precision arithmetic. These observations suggest that realizing the full potential of quantized inference requires co-designed architectures that minimize non-linear operation overhead and reduce precision-conversion frequency—a promising direction for future work.

\subsubsection{Architecture Selection Guidelines}
\label{sec:selection_guidelines}

The Pareto frontier provides a menu of optimal configurations for different latency budgets. We map representative latency targets to application domains in~\autoref{tab:applications}, providing practitioners with actionable guidance for architecture selection.

\begin{table}[h]
\vspace{-0.1cm}
\centering
\caption{Latency requirements for representative edge deployment scenarios.}
\label{tab:applications}
\small
\begin{tabular}{lll}
\toprule
Application & Latency Target & Rationale \\
\midrule
Embodied AI & $<$20 ms (decode) & Real-time interaction \\
Smart Home & $<$500 ms (total) & Conversational response \\
Autonomous Driving & $<$100 ms (total) & Safety-critical decisions \\
Private Serving & $<$2 s (total) & Quality-focused, on-device \\
\bottomrule
\end{tabular}
\end{table}

To select an architecture for a target application, practitioners first identify the latency budget dictated by system requirements, then determine the relevant optimization objective based on workload characteristics. The appropriate Pareto frontier is consulted to locate the configuration operating at the target latency, which by construction achieves the lowest attainable loss within that budget. The corresponding architectural parameters can then be directly read off and deployed, just like the different region in~\autoref{fig:pareto_app}.

\begin{figure}[h]
    \centering
    \includegraphics[width=0.6\textwidth]{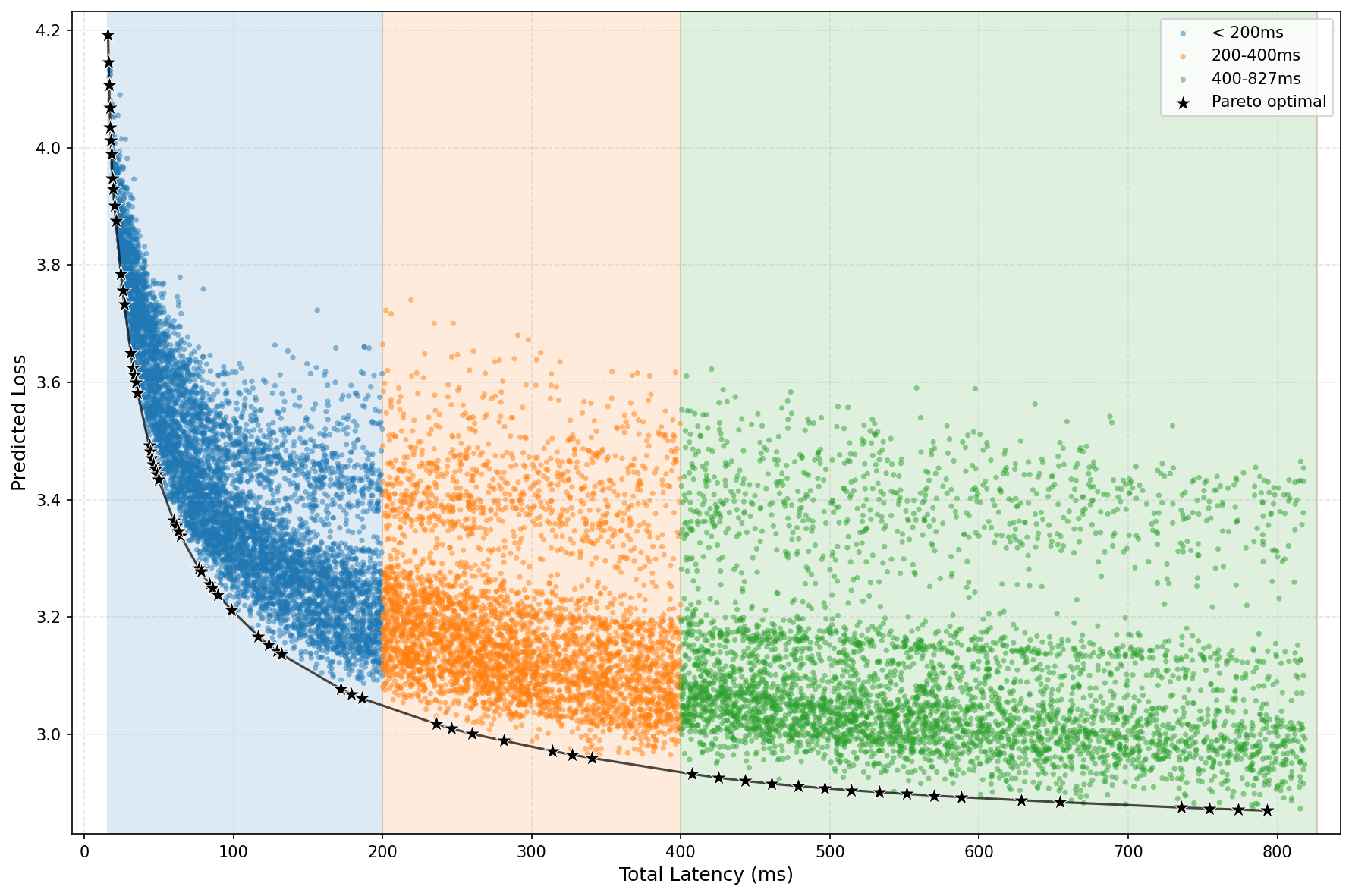}
    \caption{Different applications live in completely different regions of the Pareto frontier.}
    \label{fig:pareto_app}
\end{figure}

\subsubsection{Architecture Parameter Evolution}
\label{sec:arch_evolution}

\begin{figure}[t]
    \centering
    \includegraphics[width=0.9\textwidth]{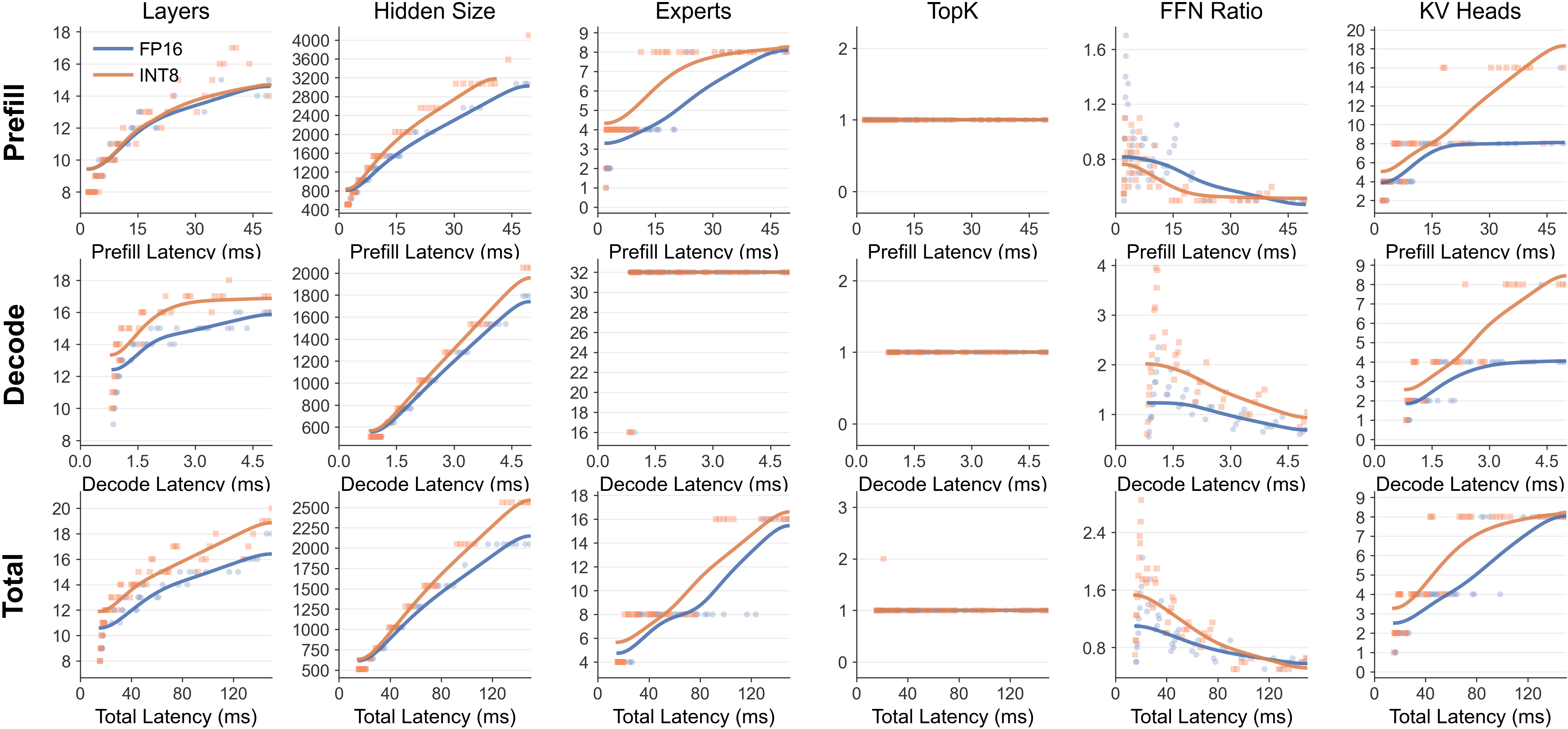}
    \caption{Architecture parameter evolution along Pareto frontiers under prefill-optimized (top), decode-optimized (middle), and total latency-optimized (bottom) objectives, comparing FP16 and INT8 precision. As the latency budget increases, optimal configurations exhibit systematic shifts in depth, width, expert count, and FFN expansion ratio.}
    \label{fig:config_evolution}
\end{figure}

\autoref{fig:config_evolution} traces how Pareto-optimal architectures evolve as the latency budget increases across different optimization objectives. Several key patterns emerge from this analysis.

\paragraph{MoE Dominance.} Sparse MoE architectures constitute 100\% of Pareto-optimal configurations across all latency regimes. Under the batch-one constraint typical of on-device deployment, MoE models achieve superior efficiency over dense counterparts: they provide greater model capacity (total parameters) while maintaining comparable activated parameters per token, yielding better loss-per-FLOP trade-offs. This finding strongly motivates the adoption of sparse architectures for edge deployment scenarios.

\paragraph{Wide-and-Shallow Preference.} In contrast to conventional LLM designs that favor deep, narrow architectures, the Pareto-optimal configurations exhibit a distinctive ``wide-and-shallow'' pattern: depth remains relatively constrained (generally below 20 layers) while width is substantially larger than comparably-sized models. Both dimensions increase with the latency budget, but width saturates earlier at the search space upper bound, after which additional capacity is allocated to depth. This pattern suggests that under strict latency constraints, width provides more efficient loss reduction per unit latency than depth—a finding with important implications for on-device model design.

\paragraph{Phase-Dependent Expert Configuration.} The optimal MoE configuration differs markedly between prefill and decode phases, driven by their distinct computational characteristics:

\begin{itemize}[leftmargin=*, itemsep=2pt]
    \item \textit{Prefill phase}: With relatively few input tokens per expert in on-device scenarios, increasing the number of experts requires loading more parameters without proportional compute utilization, shifting the bottleneck from compute-bound to memory-bound operation and degrading hardware efficiency. Consequently, prefill-optimized configurations favor fewer experts, with the expert count increasing gradually only as the latency budget relaxes.
    
    \item \textit{Decode phase}: At batch size one, each token activates a fixed subset of experts, so increasing the total number of experts incurs negligible additional latency while substantially expanding model capacity. Decode-optimized configurations therefore favor maximizing the expert count within the search space.
    
    \item \textit{Routing strategy}: Both phases consistently prefer Top-$K$=1 routing, as activating multiple experts per token substantially increases memory bandwidth consumption during the memory-bound decode phase.
\end{itemize}

\paragraph{Balanced Configuration under Total Latency.} When optimizing for total end-to-end latency, the optimal expert count reflects a trade-off between prefill and decode contributions. Prefill-dominated workloads (long input, short output) favor fewer experts; decode-dominated workloads (short input, long generation) favor more experts. Under balanced input–output ratios typical of many practical applications, the optimal configuration converges to a moderate expert count (typically around 8), consistent with the design choices observed in recent production models~\cite{guo2025deepseek,yang2025qwen3}.

\paragraph{Compact FFN Expansion.} Notably, the optimal FFN expansion ratio under on-device constraints is substantially smaller than the conventional $4\times$ used in standard Transformer designs. In many Pareto-optimal configurations, ratios below $1\times$ emerge as viable design choices, suggesting that reallocating parameters from FFN width to other dimensions (e.g., more experts or increased model width) yields better efficiency under memory-constrained inference.

\subsubsection{Empirical Validation}
\label{sec:empirical_validation}

\begin{figure}[htbp]
    \centering
    \subfigure[Pareto frontier]{
        \centering
        \includegraphics[width=0.42\linewidth]{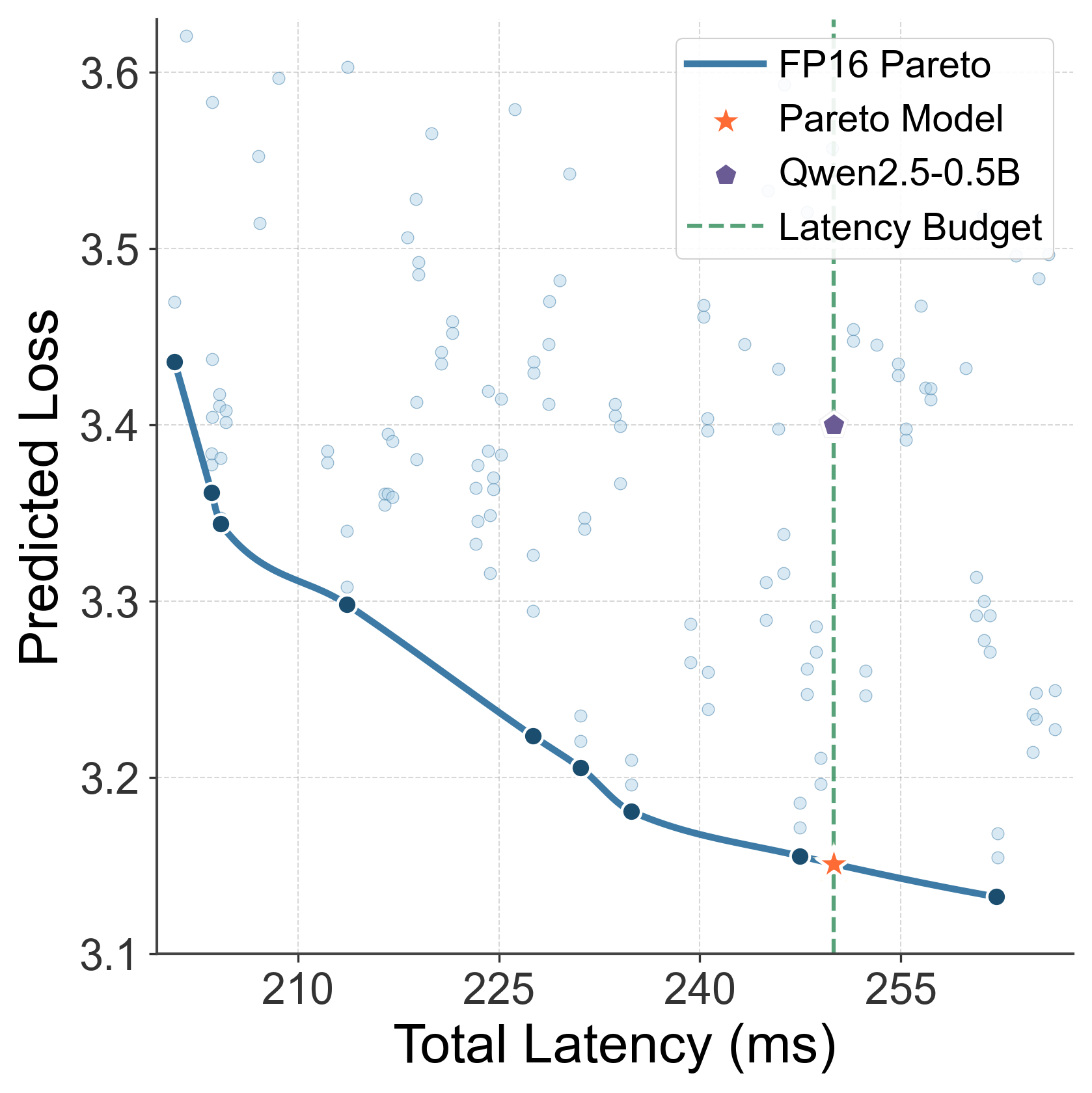}
        \label{fig:pareto_validation}
    }
    \subfigure[Training dynamics]{
        \centering
        \includegraphics[width=0.42\linewidth]{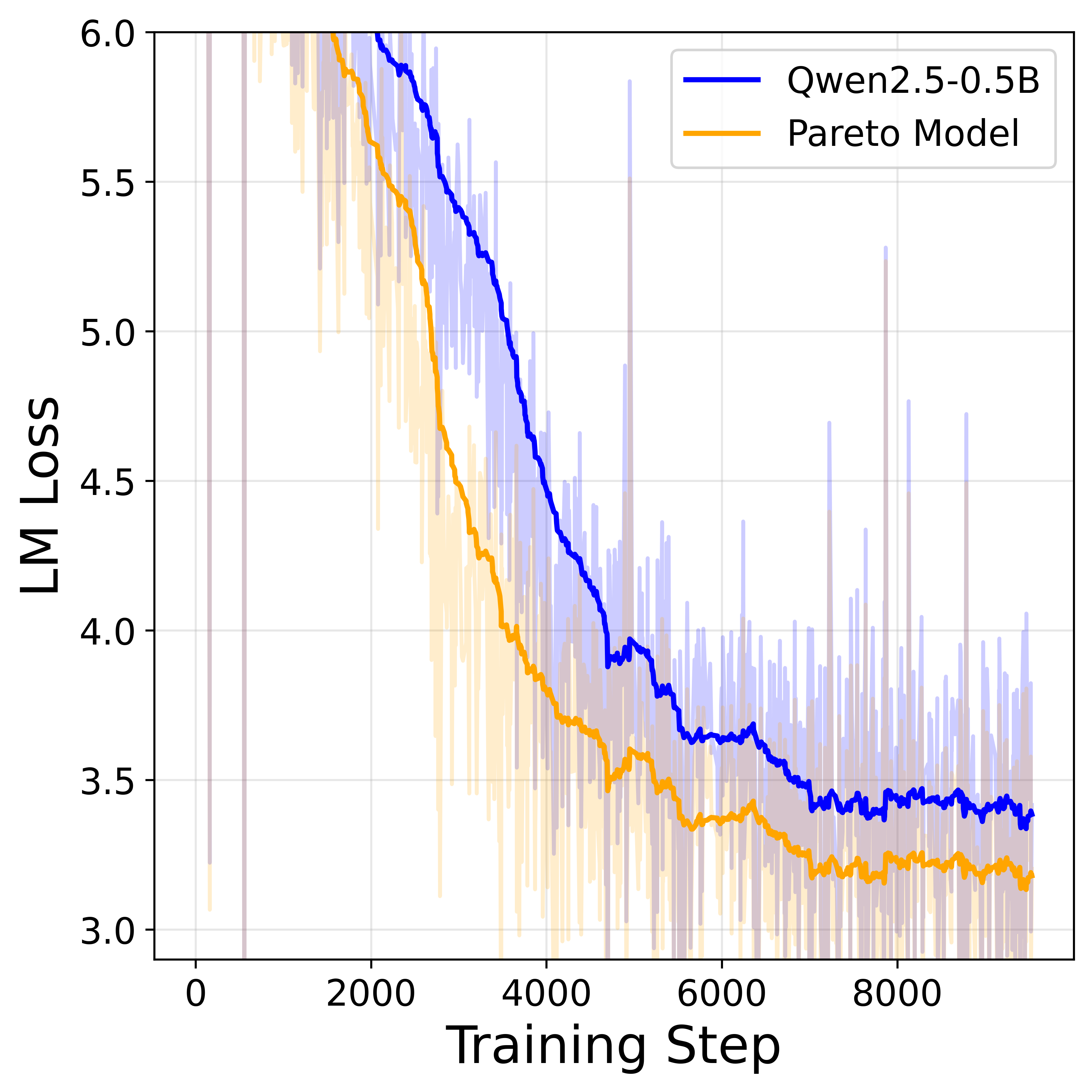}
        \label{fig:training_curves}
    }
    \caption{Empirical validation on NVIDIA Jetson Orin. (a) Pareto frontier with the co-designed model and Qwen2.5-0.5B marked. (b) Training loss curves showing faster convergence for the Pareto-optimal architecture.}
    \label{fig:empirical_validation}
\end{figure}

To validate the practical benefits of hardware-aware architecture selection, we conduct an empirical comparison against an existing production model. Using vLLM, we first measure the inference latency of Qwen2.5-0.5B on the target hardware (NVIDIA Jetson Orin), then identify a Pareto-optimal architecture from our framework that matches this measured latency, as shown in \autoref{fig:pareto_validation}. Both models are trained using an identical data mixture and optimization protocol to ensure fair comparison.


\autoref{fig:training_curves} shows that the co-designed architecture achieves consistently lower training loss throughout optimization, indicating better utilization of model capacity under the same computational budget. For downstream evaluation, we measure perplexity on WikiText-2 after training: the co-designed architecture achieves \textbf{19.42\%} lower perplexity compared to Qwen2.5-0.5B (50.88 vs.\ 63.14). This substantial improvement at equivalent inference latency demonstrates that hardware-aware architecture selection yields measurable quality gains without sacrificing deployment efficiency, validating the practical utility of the \nas{} framework.
\subsubsection{Summary of Findings}
\label{sec:findings_summary}

We summarize the key findings from our Pareto analysis:

\begin{itemize}[leftmargin=*, itemsep=2pt]
    \item \textbf{Sparse architectures dominate.} MoE configurations constitute 100\% of Pareto-optimal designs under on-device batch-one inference, providing superior capacity-efficiency trade-offs.
    
    \item \textbf{Wide-and-shallow designs are preferred.} Optimal architectures are wider and shallower than conventional designs at equivalent latency, with width providing more efficient loss reduction under tight constraints.
    
    \item \textbf{Phase-specific expert configuration.} Prefill and decode phases demand opposing expert configurations; total-latency optimization requires balancing both contributions.
    
    \item \textbf{Compact FFN expansion.} The optimal FFN expansion ratio is substantially smaller than the conventional $4\times$, with ratios below $1\times$ emerging as viable choices.
    
    \item \textbf{Quantization helps but sub-linearly.} INT8 quantization consistently improves the Pareto frontier, though gains are sub-linear due to non-linear operation and precision-conversion overhead.
    
    \item \textbf{No universal optimal architecture.} Optimal designs are hardware- and workload-specific; architectures do not transfer across platforms or deployment scenarios.
\end{itemize}

These observations provide actionable guidance for practitioners designing on-device LLMs, while the complete Pareto frontier enables precise architecture selection for specific deployment constraints. The \nas{} framework and trained model checkpoints will be released to facilitate further research in hardware-aware neural architecture design.

\section{Theoretical Framework for Hardware-Aware Architecture Optimization}
\label{sec:theory}

\subsection{From Empirical Search to Principled Optimization}
\label{sec:theory:motivation}

\autoref{sec:pareto} empirically discovered Pareto frontiers through large-scale search over \obsarch{} architectures. While effective, this approach raises fundamental questions: Can we predict optimal architectures without exhaustive search? What structural principles govern the Pareto frontier? How do solutions generalize to new hardware platforms?

This section addresses these questions by developing a theoretical framework that derives closed-form solutions for optimal architectures under different hardware constraint regimes. Rather than treating Pareto-optimal designs as empirical outcomes, we formalize architecture selection as an explicit constrained optimization problem.

\textbf{Key Insight.} Different hardware constraint regimes induce qualitatively distinct optimal solutions, particularly in how sparsity (MoE activation rate $\rho$) should be allocated. This explains why certain architectural patterns consistently emerge on the empirical Pareto frontier.

\subsection{Problem Formulation and Constraint Types}
\label{sec:theory:formulation}

We formalize the hardware co-design problem as:
\begin{equation}
\begin{aligned}
    \min_{\bm{\theta}} \quad & \hat{L}(\bm{\theta}) \\
    \text{subject to} \quad & \hat{T}(\bm{\theta}; H, W) \leq T_{\text{lat}}, \quad \hat{M}(\bm{\theta}) \leq M_{\text{budget}}
\end{aligned}
\label{eq:unified-opt}
\end{equation}
where $\bm{\theta} = (l, d, r, \rho, \text{gqa})$ denotes depth, width, FFN ratio, activation rate, and GQA ratio.

Building on roofline analysis (\autoref{app:problem_formulation}), we identify three constraint types. The \textbf{prefill constraint} (compute-bound) takes the form $l \cdot \xi_F \cdot d^2 \leq \bar{F}_p$ where $\xi_F = 4 + 4/\text{gqa} + 6r$. The \textbf{decode constraint} (bandwidth-bound) includes both weight loading and KV-cache access: $l \cdot (\xi_W^{\text{dec}} d^2 b_w + 2\bar{S} d b_{kv}/\text{gqa}) \leq \bar{M}_d$ where $\xi_W^{\text{dec}} = 2 + 2/\text{gqa} + 3r$. The \textbf{memory constraint} (storage-bound) accounts for all model parameters: $l \cdot \xi_W^{\text{all}} \cdot d^2 \cdot b_w \leq M_{\text{budget}}$ where $\xi_W^{\text{all}} = 2 + 2/\text{gqa} + 3r/\rho$. Here $\bar{S} = S_{\text{in}} + (S_{\text{out}}+1)/2$ denotes the average context length during decoding.

\textbf{Key Observation.} The activation rate $\rho$ appears only in $\xi_W^{\text{all}}$, reflecting that sparsity affects storage but not per-token computation. This asymmetry drives our main theoretical results.

\subsection{Optimal Activation Rate Across Constraint Regimes}
\label{sec:theory:activation}

We characterize optimal activation rates $\rho^*$ for three canonical regimes: latency-only (inference speed-limited with ample memory), memory-only (storage-limited with sufficient compute), and dual-constrained (tightly coupled hardware limits). These regimes naturally arise in different deployment scenarios: edge devices are often memory-constrained, automotive platforms are typically latency-constrained, and embedded systems are frequently dual-constrained.

\begin{theorem}[Latency-Constrained Regime]
\label{thm:latency-only}
When only latency constraints are active (memory unconstrained):
\begin{equation}
    \rho^* = \rho_{\min}
    \label{eq:rho-latency}
\end{equation}
\end{theorem}

\begin{proof}[Proof Sketch]
Since $\partial \xi_F / \partial \rho = \partial \xi_W^{\text{dec}} / \partial \rho = 0$, the Lagrangian gradient $\partial \mathcal{L} / \partial \rho = \partial \hat{L} / \partial \rho > 0$ everywhere. Thus $\rho^*$ occurs at the boundary. Full proof in~\autoref{app:D1} and~\autoref{app:P1}.
\end{proof}

\textbf{Interpretation.} Under latency constraints, MoE sparsity provides a ``free lunch'': reducing $\rho$ (activating fewer experts) decreases loss without increasing per-token latency, since only $K$ experts are computed regardless of total pool size $E$. The optimal strategy is therefore to maximize sparsity (minimize $\rho$) within the fixed latency budget. For latency-critical applications such as autonomous driving with sub-50ms requirements, this suggests preferring top-1 routing and increasing the total expert count as much as memory permits.

\begin{theorem}[Memory-Constrained Regime]
\label{thm:memory-only}
When only memory constraint is active (latency unconstrained):
\begin{equation}
    \rho^* = \left[ \frac{\alpha_r \kappa_d}{(\alpha_\rho - \alpha_r) \kappa_\rho} \right]^{1/\alpha_\rho} \cdot d^{(\beta_1 - \beta_2)/\alpha_\rho}
    \label{eq:rho-memory}
\end{equation}
valid when $\alpha_\rho > \alpha_r$.
\end{theorem}

\begin{proof}[Proof Sketch]
Eliminating the Lagrange multiplier from KKT conditions for $\rho$ and $r$ yields an algebraic relation. Substituting the loss function and solving gives~\autoref{eq:rho-memory}. Complete derivation in~\autoref{app:D2} and~\autoref{app:P2} .
\end{proof}

\begin{corollary}[Width-Sparsity Scaling Law]
\label{cor:width-sparsity}
Under memory constraints: $\rho^* \propto d^{(\beta_1 - \beta_2)/\alpha_\rho}$. With fitted exponents ($\beta_1 \approx -0.33$, $\beta_2 \approx 0.97$, $\alpha_\rho \approx 1.09$), this implies \emph{wider models should use sparser MoE}.
\end{corollary}

\textbf{Interpretation.} Memory-constrained systems face a fundamental trade-off: storing all $E$ experts costs proportional to $1/\rho$, but increased sparsity provides capacity gains that benefit wider models more than narrower ones.~\autoref{eq:rho-memory} characterizes the optimal balance point. The width-sparsity coupling arises because the sparsity capacity term in the loss has negative width exponent $\beta_1 < 0$ (making sparsity more valuable for wider models), while the base capacity term has positive exponent $\beta_2 > 0$ (meaning dense capacity scales favorably with width). For practical deployment on memory-limited devices with 4--8~GB DRAM, a 2B-parameter model with $d \approx 2048$ should use $\rho \approx 0.15$ (e.g., $K=2, E=16$), while a 500M-parameter model with $d \approx 1024$ should use denser MoE at $\rho \approx 0.25$.

\begin{theorem}[Dual-Constrained Regimes]
\label{thm:dual-constrained}
When both latency and memory constraints are active, the optimal activation rate depends on which latency phase is limiting.

\textbf{(a) Prefill + Memory:}
\begin{equation}
    \rho^* = \frac{3 \eta_p b_w r}{\alpha_{\text{attn}}(2 - \eta_p b_w) + 6r}, \quad \eta_p = \bar{F}_p / M_{\text{budget}}, \quad \eta_p b_w < 2
    \label{eq:rho-dual-prefill}
\end{equation}

\textbf{(b) Decode + Memory:}
\begin{equation}
    \rho^* = \frac{3r}{\xi_M^* - \alpha_{\text{attn}}}, \quad \xi_M^* = \frac{(\alpha_{\text{attn}} + 3r) + \sqrt{(\alpha_{\text{attn}} + 3r)^2 + 4\eta\delta}}{2\eta}
    \label{eq:rho-dual-decode}
\end{equation}
where $\eta = \bar{M}_d / M_{\text{budget}}$, $\alpha_{\text{attn}} = 2 + 2/\text{gqa}$, and $\delta = 2\bar{S} b_{kv} / (\text{gqa} \cdot d \cdot b_w)$ is the KV-cache correction. Here $\xi_M^* \triangleq \xi_W^{\mathrm{all}}\big|_{\rho=\rho^*}$ denotes the equilibrium value of the storage coefficient (defined in~\autoref{app:problem_formulation}) under the dual constraint, obtained by solving a quadratic system.
\end{theorem}

\begin{proof}[Proof Sketch]
Constraint compatibility requires $\xi_F / (b_w \xi_W^{\text{all}}) = \eta_p$ (prefill case) or the decode analog. Solving these algebraic equations yields the stated forms. The decode case involves a quadratic due to the KV-cache term.
\end{proof}

\textbf{Comparison.} The prefill+memory case admits a simple closed form, while decode+memory requires solving a quadratic due to KV-cache coupling. The key difference is that the decode constraint includes a term proportional to $\bar{S}/\text{gqa}$ absent in the prefill case, creating stronger coupling between gqa and the latency budget. For systems with tight coupling between latency and memory such as automotive SoCs, practitioners should compute the constraint ratio $\eta$ or $\eta_p$ and apply the corresponding formula, verifying that $\rho_{\min} \leq \rho^* \leq 1$.

\autoref{tab:theorem-proof-map} provides a cross-reference between the theoretical results presented in this section and their complete derivations in the appendices.

\begin{table}[h]
\centering
\caption{Cross-reference of theoretical results and detailed proofs.}
\label{tab:theorem-proof-map}
\small
\begin{tabular}{lll}
\toprule
\textbf{Result} & \textbf{Main Result} & \textbf{Detailed Proof} \\
\midrule
\autoref{thm:latency-only} & $\rho^* = \rho_{\min}$ & \autoref{app:D1} (D1), \autoref{app:P1} (P1) \\
\autoref{thm:memory-only}, \autoref{cor:width-sparsity} & $\rho^* \propto d^{(\beta_1-\beta_2)/\alpha_\rho}$ & \autoref{app:D2} (D2), \autoref{app:P2} (P2) \\
\autoref{eq:rho-dual-prefill} in~\autoref{thm:dual-constrained} & $\rho^*$ prefill+memory & \autoref{app:P3} (P3) \\
\autoref{eq:rho-dual-decode} in~\autoref{thm:dual-constrained} & $\rho^*$ decode+memory & \autoref{app:D3} (D3) \\
\autoref{eq:r-prefill}--\autoref{eq:gqa-prefill} & $r^*,\;\mathrm{gqa}^*$ closed forms & \autoref{app:summary} \\
\bottomrule
\end{tabular}
\end{table}

\subsection{Optimal Depth, FFN Ratio, and GQA}
\label{sec:theory:other}

The optimal depth always saturates the active constraint, taking the form $l^* = \bar{F}_p / (\xi_F d^2)$ for prefill, $l^* = \bar{M}_d / (\xi_W^{\text{eff}} d^2 b_w)$ for decode, or $l^* = M_{\text{budget}} / (\xi_W^{\text{all}} d^2 b_w)$ for memory constraints. This implies a fundamental depth-width trade-off: $l^* \propto d^{-2}$ at fixed budget, explaining the inverse scaling behavior observed along empirical Pareto frontiers.

The optimal FFN ratio and GQA ratio follow structured closed forms. We define the aggregate loss gradient $\tilde{D} \triangleq \kappa_\rho \rho^{\alpha_\rho} d^{\beta_2 - \beta_1} + \kappa_d$, which combines the sparsity and base capacity contributions. Under prefill latency constraints, the solutions are:
\begin{align}
    r^* &= \left[ \frac{\alpha_r \tilde{D}}{6 \alpha_l \kappa_l} \cdot \frac{\bar{F}_p^{\alpha_l}}{\xi_F^{\alpha_l - 1} d^{2\alpha_l + \beta_2}} \right]^{1/(\alpha_r + 1)} \label{eq:r-prefill} \\
    \mathrm{gqa}^* &= \left[ \frac{4 \alpha_l \kappa_l}{\alpha_m \kappa_m} \cdot \frac{\xi_F^{\alpha_l - 1} d^{2\alpha_l + \alpha_m}}{\bar{F}_p^{\alpha_l}} \right]^{1/(\alpha_m + 1)} \label{eq:gqa-prefill}
\end{align}
Other constraint regimes share the same structural form with modified coefficients and budget terms. Under decode latency constraints, the per-layer constraint coefficient generalizes to $\Gamma \triangleq \xi_W^{\mathrm{dec}} d^2 b_w + 2\bar{S} d b_{kv}/\mathrm{gqa}$, which accounts for KV-cache bandwidth; the $r^*$ coefficient changes from $1/6$ to $1/3$ (since $\partial \xi_W^{\mathrm{dec}}/\partial r = 3$ vs.\ $\partial \xi_F/\partial r = 6$), and the $\mathrm{gqa}^*$ formula acquires an additional KV-cache coupling factor $(d^2 b_w + \bar{S} d b_{kv})$. Under memory constraints, an extra $\rho^*$ factor enters the $r^*$ numerator. \autoref{tab:coefficients-summary} summarizes the key differences; complete formulas for all regimes are provided in~\autoref{app:summary}.

\begin{table}[h]
\centering
\caption{Coefficient comparison across constraint regimes. $C_r$ and $C_g$ denote the constraint-derivative coefficients for $r^*$ and $\mathrm{gqa}^*$, respectively.}
\label{tab:coefficients-summary}
\begin{tabular}{@{}lccccc@{}}
\toprule
\textbf{Regime} & \textbf{$\rho^*$} & \textbf{$C_r$} & \textbf{$C_g$} & \textbf{Budget} & \textbf{Constraint Coeff.} \\
\midrule
Prefill Latency & $\rho_{\min}$ & 6 & 4 & $\bar{F}_p$ & $\xi_F$ \\
Decode Latency & $\rho_{\min}$ & 3 & 2 & $\bar{M}_d$ & $\Gamma$ \\
Memory & \autoref{eq:rho-memory} & $3/\rho$ & 2 & $M_{\mathrm{budget}}$ & $\xi_W^{\mathrm{all}} d^2 b_w$ \\
\midrule
Prefill + Mem & \autoref{eq:rho-dual-prefill} & \multicolumn{4}{c}{see~\autoref{app:summary}} \\
Decode + Mem & \autoref{eq:rho-dual-decode} & \multicolumn{4}{c}{see~\autoref{app:summary}} \\
\bottomrule
\end{tabular}
\end{table}

The 2$\times$ difference in coefficients between prefill and decode (e.g., $C_r = 6$ vs.\ $3$) arises from the fundamental relation $\xi_F = 2\xi_W^{\mathrm{dec}}$, which reflects the FLOPs-to-memory-access ratio: each multiply-accumulate operation counts as 2 FLOPs but requires loading each weight parameter only once. Note that the decode $C_g = 2$ in~\autoref{tab:coefficients-summary} reflects only the weight-loading contribution; the full decode GQA derivative additionally includes a KV-cache correction term $2\bar{S} b_{kv}/(d \cdot b_w)$ (see~\autoref{app:D1} for the complete expression). This structural asymmetry has important practical implications: prefill-optimized models should use smaller FFN ratios and larger GQA values compared to decode-optimized models at equivalent performance levels.

\subsection{Design Principles and Practical Guidelines}
\label{sec:theory:insights}

\subsubsection{Key Structural Insights}

Four structural results emerge from the theoretical analysis. First, memory-constrained solutions exhibit scenario independence: the optimal parameters $(\rho^*, l^*, r^*, \text{gqa}^*)$ are identical for prefill and decode phases, since memory constraints concern model storage rather than inference dynamics. Second, prefill versus decode constraints induce a coefficient asymmetry, with 2$\times$ differences in FFN and GQA coefficients stemming from $\xi_F = 2\xi_W^{\text{dec}}$. Third, decode constraints exhibit KV-cache coupling through a term proportional to $\bar{S}/\text{gqa}$ that creates sequence-length dependence absent in prefill constraints. Fourth, the width-sparsity scaling law $\rho^* \propto d^{-1.19}$ implies that doubling model width should reduce activation rate by approximately $2.3\times$, providing a principled basis for allocating sparsity in memory-limited deployments.

\subsubsection{Actionable Design Guidelines}

\paragraph{Sparsity Allocation Strategy.}
The optimal sparsity allocation depends critically on the active constraint regime. For latency-bound systems, practitioners should maximize sparsity by setting $\rho = \rho_{\min}$ (typically top-1 routing with $K=1$). For memory-bound systems, the width-sparsity scaling law (\autoref{eq:rho-memory}) provides the principled allocation: wider models require sparser MoE configurations to balance capacity gains against storage costs. For dual-constrained systems, the regime-specific formulas (\autoref{eq:rho-dual-prefill} or~\autoref{eq:rho-dual-decode}) should be applied after computing the constraint ratio $\eta$ or $\eta_p$ to determine which formula is appropriate.

\paragraph{Depth-First Budget Allocation.}
The $l^* \propto d^{-2}$ relationship suggests a systematic allocation strategy. Practitioners should first select a target width $d$ based on the parameter budget and width-sparsity law, then compute optimal depth $l^*$ to saturate the active constraint using the appropriate formula from~\autoref{sec:theory:other}. If the computed $l^*$ exceeds the architectural search space upper bound (e.g., $l > 48$ layers), the width $d$ should be reduced iteratively until a feasible depth is obtained. This depth-first strategy aligns with the empirical observation that depth grows monotonically along Pareto frontiers until reaching architectural limits.

\paragraph{Phase-Aware Parameter Tuning.}
The coefficient asymmetries in~\autoref{tab:coefficients-summary} translate directly into optimization strategies. For prefill-dominant workloads such as long-context question answering, models should employ smaller FFN ratios (exploiting the $1/6$ coefficient versus $1/3$ in decode) and larger GQA values (more KV heads) to amortize projection costs, while ignoring KV-cache overhead in the optimization. Conversely, for decode-dominant workloads such as chatbots and code generation, models should use larger FFN ratios and carefully balance GQA against KV-cache bandwidth, as the decode $\mathrm{gqa}^*$ formula includes a sequence-length-dependent KV-cache correction (see~\autoref{app:summary}). For balanced workloads with mixed prefill and decode phases, practitioners should optimize for total end-to-end latency or defer to memory-constrained formulas if storage is the primary limitation.

\paragraph{Generalization to New Hardware Platforms.}
One primary motivation for this theoretical framework is enabling efficient architecture search on new hardware platforms without repeating exhaustive empirical evaluation. For a new platform with specifications $(\pi_H, \beta_H, M_{\text{budget}})$, the deployment workflow proceeds as follows. First, measure the hardware parameters: peak compute $\pi_H$, sustained memory bandwidth $\beta_H$, and available memory $M_{\text{budget}}$. Second, define application requirements including target latency budgets $T_{\text{lat}}^{\text{pre}}, T_{\text{lat}}^{\text{dec}}$ and workload configuration $(B, S_{\text{in}}, S_{\text{out}})$. Third, compute normalized budgets $\bar{F}_p = T_{\text{lat}}^{\text{pre}} \cdot \pi_H / (BS_{\text{in}})$ and $\bar{M}_d = T_{\text{lat}}^{\text{dec}} \cdot \beta_H / S_{\text{out}}$. Fourth, determine the active constraint regime by computing ratios $\eta_p = \bar{F}_p / M_{\text{budget}}$ and $\eta = \bar{M}_d / M_{\text{budget}}$: if either ratio is much less than 1, the system is memory-constrained; if much greater than 1, it is latency-constrained; otherwise, it is dual-constrained. Fifth, apply the corresponding theorem to predict optimal parameters $\bm{\theta}^*$ and round to feasible discrete values. Finally, validate predictions with 3--5 small-scale training runs (1--2B tokens each) to measure actual latency and refine if systematic bias is observed.

This workflow reduces architecture selection time from months (full empirical search) to under one week (theoretical prediction plus small-scale validation), as demonstrated in our deployment case study. As a concrete example, consider deploying on a new edge device with 10~TOPS compute, 50~GB/s bandwidth, 4~GB memory, and a target decode latency below 100ms for single-token generation. Computing $\bar{M}_d = 0.1 \times 50 / 10 = 0.5$~GB and the ratio $\eta = 0.5 / 4 = 0.125 < 1$ identifies this as a memory-constrained regime. Applying ~\autoref{thm:memory-only} for width $d = 1024$ predicts $\rho^* \approx 0.20$. Training a targeted 20-layer, 1024-width MoE model with $K=2, E=10$ validates these predictions without evaluating thousands of candidate architectures.

\subsubsection{Limitations and Future Extensions}

The theoretical framework rests on three key assumptions that bound its applicability. First, the fitted loss scaling law (\autoref{eq:loss}) is based on \trainarch{} architectures trained for 10B tokens; extrapolation to significantly different training budgets or data distributions may reduce prediction accuracy and should be validated empirically. Second, the latency model assumes idealized roofline behavior (\autoref{eq:roofline}), whereas real systems exhibit kernel launch overhead, cache effects, and operator fusion that may cause deviations of 10--20\% from theoretical predictions. Third, the framework assumes standard transformer components including attention, FFN, and MoE; extensions to hybrid architectures such as SSM-Transformer blends or linear attention mechanisms would require re-deriving constraint forms and may exhibit qualitatively different scaling behaviors.

Future work could address these limitations by incorporating training dynamics (learning rate schedules, optimizer state) into the loss model, developing more refined latency models that account for operator fusion and system-level effects, extending the theory to hybrid architectures and emerging attention mechanisms, and validating the framework across a broader range of hardware platforms including TPUs and specialized AI accelerators. Nonetheless, the current framework represents a significant advance toward principled, hardware-aware LLM architecture design grounded in explicit optimization theory rather than pure empirical search.

\subsection{Summary}

This section developed a comprehensive theoretical framework for hardware-aware architecture optimization. \autoref{thm:latency-only} through~\autoref{thm:dual-constrained} characterize optimal activation rates $\rho^*$ across latency-only, memory-only, and dual-constrained regimes, revealing that different hardware constraints induce qualitatively different optimal solutions. The width-sparsity scaling law (\autoref{cor:width-sparsity}) establishes that $\rho^* \propto d^{-1.19}$, providing a principled basis for allocating sparsity in memory-constrained settings. Optimal depth, FFN ratio, and GQA configurations expose structural asymmetries between prefill-optimized and decode-optimized architectures, with 2$\times$ coefficient differences arising from the fundamental FLOPs-to-memory-access ratio. The derived design principles enable rapid architecture selection on new hardware platforms, reducing deployment time from months to under one week. The theoretical predictions align closely with empirical Pareto frontiers discovered in~\autoref{sec:pareto}, validating the framework while providing deeper insight into why certain architectural patterns emerge as optimal. Complete proofs and detailed parameter formulas are provided in~\autoref{app:problem_formulation} through~\autoref{app:summary}.

\section{Conclusion}

We present a hardware-aligned framework that connects transformer architecture to model quality and end-to-end inference efficiency via equivalent-parameter scaling and hardware-aware latency modeling. Evaluating \obsarch{} architectures, we identify clear structural principles and Pareto-optimal regimes that govern the trade-off between loss, latency, and roofline efficiency.
Our results show that effective LLM deployment on edge and embedded accelerators requires explicit hardware-model co-design. Beyond heuristic selection, the proposed hardware co-design scaling law reveals a stable and predictable relationship between architecture and hardware constraints, enabling principled extrapolation of optimal designs across deployment regimes.
\newpage

\bibliographystyle{unsrt}
\bibliography{main}

\newpage
\addtocontents{toc}{\protect\setcounter{tocdepth}{1}}
\appendix

\section{Architecture Search Space}
\label{app:search_space}

\autoref{tab:scaling_space_app} specifies the architectural hyperparameters explored in our scaling law study.

\begin{table}[h]
\caption{Architecture search space for scaling law fitting.}
\label{tab:scaling_space_app}
\small
\centering
\begin{tabular}{ll}
\toprule
Parameter & Values \\                                                                              
\midrule
Depth $l$ & $\{4, 8, 12, 16, 20, 24, 28, 32\}$ \\
Width $d$ & $\{768, 1024, 1280, 1536, 1792, 2048, 2304, 2560, 3072\}$ \\
MoE $(E, K)$ & $\{(1,1), (8,1), (8,2), (16,1), (16,2)\}$ \\
GQA $n_{kv}$ & $\{1, 2, 4, 8, n_h\}$ \\
\bottomrule
\end{tabular}
\end{table}

The search space jointly covers depth (4–32 layers), width (768–3072 hidden dimensions), MoE configurations (dense to 16 experts with Top-1/Top-2 routing), and grouped-query attention settings. This design ensures coverage of both dense and sparse architectures while avoiding degenerate or ill-conditioned regimes.

\section{Pre-training Details}
\label{app:pretrain_details}

All \trainarch{} model configurations are trained under identical conditions to ensure fair comparison:

\paragraph{Training Data.} Each model is trained on 10B tokens from a mixture of general corpus, mathematics, and code data. This budget is sufficient to observe scaling behavior while remaining computationally tractable. The training corpus will be released upon publication.

\paragraph{Optimization.} We use the AdamW optimizer with $\beta_1 = 0.9$, $\beta_2 = 0.95$, and weight decay $0.01$. The learning rate follows a cosine decay schedule from $1 \times 10^{-4}$ to $1 \times 10^{-6}$, with linear warmup over the first 0.2\% of training steps. QK-Norm is applied to stabilize MoE pre-training. All experiments use a global batch size of 256.

\paragraph{Evaluation.} Model performance is measured by validation loss on a held-out set of approximately 1B tokens, averaged over the final 1,000 optimization steps to reduce variance. We additionally report perplexity on WikiText-2 for downstream evaluation.

\paragraph{Computational Cost.} Each configuration is trained on 8 NVIDIA H200 GPUs for approximately 10 hours, totaling $\trainarch{} \times 8 \times 10 = 13{,}600$ GPU-hours for the full study.


\section{Scaling Law Coefficients}
\label{app:scaling_coefficients}

The fitted parametric loss model takes the following form:

\begin{equation}
\label{eq:fitting_obj_app}
\hat{\mathcal{L}}(\bm{\theta})
=
\underbrace{\frac{\kappa_l}{l^{\alpha_l}}}_{\text{depth}}
+
\underbrace{\frac{\kappa_\rho \rho^{\alpha_\rho}}{r^{\alpha_r} d^{\beta_1}}}_{\text{sparsity-width}}
+
\underbrace{\frac{\kappa_d}{r^{\alpha_r} d^{\beta_2}}}_{\text{capacity}}
+
\underbrace{\frac{\kappa_m}{d_m^{\alpha_m}}}_{\text{KV-cache}}
+ \mathcal{L}_\infty
\end{equation}

\autoref{tab:fitted_params_app} reports the fitted coefficients. The resulting concrete form is:

\begin{equation}
\hat{\mathcal{L}}(\bm{\theta})
=
\frac{9.96}{l^{1.63}}
+
\frac{0.031\,\rho^{1.09}}{r^{0.17}\,d^{-0.33}}
+
\frac{500}{r^{0.17}\,d^{0.97}}
+
\frac{0.20}{d_m^{0.05}}
+ 2.53
\end{equation}

\begin{table}[h]
\centering
\caption{Fitted scaling law coefficients.}
\label{tab:fitted_params_app}
\small
\begin{tabular}{lccl}
\toprule
Term & Coefficient & Exponent & Interpretation \\
\midrule
Depth & $\kappa_l = 9.96$ & $\alpha_l = 1.63$ & Strong depth dependence \\
Sparsity & $\kappa_\rho = 0.031$ & $\alpha_\rho = 1.09$, $\beta_1 = -0.33$ & Width-sparsity coupling \\
Capacity & $\kappa_d = 500$ & $\beta_2 = 0.97$ & Base capacity scaling \\
FFN ratio & -- & $\alpha_r = 0.17$ & Unified FFN scaling \\
KV-cache & $\kappa_m = 0.20$ & $\alpha_m = 0.05$ & Cache efficiency \\
\midrule
Irreducible & \multicolumn{3}{c}{$\mathcal{L}_\infty = 2.53$} \\
\bottomrule
\end{tabular}
\end{table}

The exponents reveal several insights: (1) depth exhibits the strongest scaling ($\alpha_l = 1.63$), indicating high sensitivity to layer count; (2) width and sparsity are coupled ($\beta_1 = -0.33$): the sparsity-related loss penalty grows with width, implying that wider models require greater sparsity (lower $\rho$) to remain competitive; (3) FFN expansion ratio contributes modestly ($\alpha_r = 0.17$); and (4) KV-cache configuration has minimal impact on loss ($\alpha_m = 0.05$) but significantly affects inference efficiency.


\section{Latency Modeling Details}
\label{app:latency_details}

\paragraph{Roofline-Based Prediction.} We analytically model latency using hardware roofline analysis. Each operator is classified as compute-bound or memory-bound based on its arithmetic intensity $I = \text{FLOPs} / \text{Bytes}$ relative to the hardware's compute-to-bandwidth ratio. Latency is estimated as:
\begin{equation}
T_{\text{op}} = \max\left(\frac{\text{FLOPs}}{\text{Peak Compute}}, \frac{\text{Memory Access}}{\text{Peak Bandwidth}}\right)
\end{equation}
This approach enables rapid evaluation of over 50,000 configurations in minutes, making it suitable for large-scale architecture exploration.

\paragraph{Empirical Validation.} We validate roofline predictions using vLLM with subprocess isolation to ensure accurate GPU memory accounting. Top Pareto candidates identified by analytical modeling are measured empirically to confirm their optimality.

\paragraph{Latency Scaling Behavior.} Inference latency varies with workload parameters:
\begin{itemize}[leftmargin=*, itemsep=2pt]
    \item \textbf{Sequence length:} Prefill latency scales quadratically with input length due to attention ($O(S^2)$); decode latency scales linearly due to KV-cache access ($O(S)$).
    \item \textbf{Output length:} Total latency is dominated by decode for long generations, making decode optimization critical for conversational applications.
    \item \textbf{Batch size:} Both phases benefit from batching with diminishing returns. Prefill becomes compute-bound at larger batches; decode remains memory-bound due to weight loading. For on-vehicle deployment, we focus on batch size 1.
\end{itemize}

\paragraph{Scenario-Specific Targets.} The appropriate latency optimization target depends on deployment scenario:
\begin{itemize}[leftmargin=*, itemsep=2pt]
    \item \textit{Interactive/streaming applications} (chatbots, speculative decoding): optimize decode latency.
    \item \textit{Long-context processing} (document QA): optimize prefill latency.
    \item \textit{Balanced workloads} (summarization): optimize total latency.
\end{itemize}
\section{Problem Formulation and Roofline Analysis}
\label{app:problem_formulation}

This section presents a comprehensive roofline analysis for decoder-only Transformers, deriving FLOPs, memory traffic, and latency models for both prefill and decode phases.

\subsection{Notation}

\begin{table}[h]
\caption{Symbol definitions for roofline analysis.}
\label{tab:notation}
\centering
\small
\begin{tabular}{cl|cl}
\toprule
\textbf{Symbol} & \textbf{Definition} & \textbf{Symbol} & \textbf{Definition} \\
\midrule
$B$ & Batch size & $S$ & Sequence length \\
$l$ & Number of layers & $d$ & Hidden dimension \\
$d_h$ & Head dimension & $n_h$ & Number of query heads \\
$n_{kv}$ & Number of KV heads & $\mathrm{gqa}$ & GQA ratio $= n_h / n_{kv}$ \\
$d_m$ & KV dimension $= d/\mathrm{gqa}$ & $r$ & FFN expansion ratio \\
$E$ & Total experts & $K$ & Active experts per token \\
$\rho$ & Activation rate $= K/E$ & $b_w$ & Bytes per weight \\
$b_a$ & Bytes per activation & $b_{kv}$ & Bytes per KV element \\
$\pi_H$ & Peak compute (FLOP/s) & $\beta_H$ & Memory bandwidth (Bytes/s) \\
$S_{\mathrm{in}}$ & Input sequence length & $S_{\mathrm{out}}$ & Output sequence length \\
\bottomrule
\end{tabular}
\end{table}

\subsection{Transformer Architecture}
\label{sec:architecture}

We consider a decoder-only Transformer with the following components per layer:

\paragraph{Multi-Head Attention.}
The attention mechanism uses $n_h$ query heads and $n_{kv}$ key-value heads (Grouped Query Attention). The hidden dimension is $d$, and the GQA ratio is $\mathrm{gqa} = n_h / n_{kv}$. The projection matrices are:
\begin{itemize}
    \item Query projection $W_Q \in \mathbb{R}^{d \times d}$
    \item Key projection $W_K \in \mathbb{R}^{d \times d/\mathrm{gqa}}$
    \item Value projection $W_V \in \mathbb{R}^{d \times d/\mathrm{gqa}}$
    \item Output projection $W_O \in \mathbb{R}^{d \times d}$
\end{itemize}

\paragraph{Feed-Forward Network.}
We consider a gated FFN (SwiGLU) with expansion ratio $r = d_{\mathrm{ffn}} / d$:
\begin{itemize}
    \item Up projection $W_{\mathrm{up}} \in \mathbb{R}^{d \times rd}$
    \item Gate projection $W_{\mathrm{gate}} \in \mathbb{R}^{d \times rd}$
    \item Down projection $W_{\mathrm{down}} \in \mathbb{R}^{rd \times d}$
\end{itemize}

\paragraph{Mixture of Experts.}
In an MoE layer, the FFN is replicated into $E$ experts, of which $K$ are activated per token. The activation rate is $\rho = K / E$.

\subsection{Per-Operator Analysis}
\label{sec:per_operator}

\subsubsection{Attention Projection Layers}

Multi-head attention uses four linear projections: Query (Q), Key (K), Value (V), and Output (O). Modern architectures employ Grouped-Query Attention (GQA)~\citep{ainslie2023gqa} to reduce K/V dimensions by a factor of \texttt{gqa}.

\paragraph{Q Projection.} 
Transform input $X \in \mathbb{R}^{B \times S \times d}$ to queries $Q \in \mathbb{R}^{B \times S \times d}$ via weight $W_Q \in \mathbb{R}^{d \times d}$.

\begin{table}[h]
\centering
\caption{Q projection costs in prefill and decode phases.}
\label{tab:q_projection_metrics}
\small
\begin{tabular}{lcc}
\toprule
\textbf{Metric} & \textbf{Prefill} & \textbf{Decode} ($S_q = 1$) \\
\midrule
FLOPs & $2BSd^2$ & $2Bd^2$ \\
Weight Load & $d^2 \cdot b_w$ & $d^2 \cdot b_w$ \\
Activation Load & $BSd \cdot b_a$ & $Bd \cdot b_a$ \\
Activation Store & $BSd \cdot b_a$ & $Bd \cdot b_a$ \\
\bottomrule
\end{tabular}
\end{table}

\paragraph{Arithmetic Intensity.} 
From~\autoref{tab:q_projection_metrics}:
\begin{itemize}
    \item \textbf{Prefill}: $\mathcal{I} \approx \frac{2BS}{b_w}$, typically compute-bound for large $BS$.
    \item \textbf{Decode}: $\mathcal{I} \approx \frac{2B}{b_w}$, typically memory-bound for small $B$.
\end{itemize}

\paragraph{K and V Projections (GQA).} 
Project to reduced dimension $d_m = d/\text{gqa}$: $X \in \mathbb{R}^{B \times S \times d} \rightarrow K, V \in \mathbb{R}^{B \times S \times d_m}$ via $W_K, W_V \in \mathbb{R}^{d \times d_m}$.

\begin{table}[h]
\centering
\caption{K and V projection costs with GQA. Each operates on dimension $d_m = d/\text{gqa}$.}
\label{tab:kv_projection_metrics}
\small
\begin{tabular}{lcc}
\toprule
\textbf{Metric (each)} & \textbf{Prefill} & \textbf{Decode} \\
\midrule
FLOPs & $2BSd^2/\text{gqa}$ & $2Bd^2/\text{gqa}$ \\
Weight Load & $d^2b_w/\text{gqa}$ & $d^2b_w/\text{gqa}$ \\
KV Cache Store & $BSd \cdot b_{kv}/\text{gqa}$ & $Bd \cdot b_{kv}/\text{gqa}$ \\
\bottomrule
\end{tabular}
\end{table}

GQA reduces K/V computation, weight memory, and critically KV-cache storage by factor \texttt{gqa} compared to Q/O projections.

\paragraph{O Projection.} 
Identical to Q projection (see~\autoref{tab:q_projection_metrics}).

\paragraph{Summary.} 
Total costs for Q, K, V, O projections:
\begin{align}
\text{FLOPs} &= 2BSd^2 \left(2 + \frac{2}{\text{gqa}}\right) \label{eq:attn_proj_flops} \\
\text{Weight Memory} &= d^2b_w \left(2 + \frac{2}{\text{gqa}}\right) \label{eq:attn_proj_weight}
\end{align}

For $\text{gqa}=1$ (standard MHA): $8BSd^2$ FLOPs, $4d^2b_w$ memory. For $\text{gqa}=8$: $\sim$44\% reduction.

\subsubsection{Attention Score Computation}

The attention mechanism computes scores through three main operations: query-key multiplication, softmax normalization, and score-value multiplication. We analyze each operation separately.

\paragraph{QK Matmul (Query-Key Attention Scores).}
For each attention head, compute attention scores by multiplying query $Q_h \in \mathbb{R}^{S_q \times d_h}$ with key transpose $K_h^T \in \mathbb{R}^{d_h \times S_{kv}}$, producing scores $\in \mathbb{R}^{S_q \times S_{kv}}$.

\begin{table}[h]
\centering
\caption{Query-key attention score computation costs. Note $d_m = d/\text{gqa}$ is the reduced dimension for K cache with grouped-query attention.}
\label{tab:qk_matmul}
\small
\begin{tabular}{lcc}
\toprule
\textbf{Metric} & \textbf{Prefill} ($S_q = S_{kv} = S$) & \textbf{Decode} ($S_q = 1$) \\
\midrule
FLOPs & $2BS^2 d$ & $2BSd$ \\
Load Q & $BSd \cdot b_a$ & $Bd \cdot b_a$ \\
Load K Cache & $BSd_m \cdot b_{kv}$ & $BSd_m \cdot b_{kv}$ \\
Store Scores & $B n_h S^2 \cdot b_a$ & $B n_h S \cdot b_a$ \\
\bottomrule
\end{tabular}
\end{table}

Prefill computes a full $S \times S$ attention matrix per head ($O(S^2)$ complexity), while decode only computes scores for one new token against $S$ cached keys ($O(S)$ complexity).

\paragraph{Softmax Normalization.}
Attention scores are normalized using softmax, requiring approximately 5 operations per element: finding maximum (for numerical stability), subtraction, exponentiation, summation, and division.

\begin{table}[h]
\centering
\caption{Softmax normalization costs over $B$ batches, $n_h$ heads, each processing $S_q \times S_{kv}$ scores.}
\label{tab:softmax}
\small
\begin{tabular}{lcc}
\toprule
\textbf{Metric} & \textbf{Prefill} ($S_q = S$) & \textbf{Decode} ($S_q = 1$) \\
\midrule
FLOPs & $\approx 5 B n_h S^2$ & $\approx 5 B n_h S$ \\
Memory Load & $B n_h S^2 \cdot b_a$ & $B n_h S \cdot b_a$ \\
Memory Store & $B n_h S^2 \cdot b_a$ & $B n_h S \cdot b_a$ \\
\bottomrule
\end{tabular}
\end{table}

Softmax FLOPs are typically negligible compared to matrix multiplications, contributing $O(n_h S^2)$ versus $O(Sd^2)$ for QK matmul when $d \gg n_h$.

\paragraph{Score-Value Matmul (Weighted Aggregation).}
Normalized attention scores $\in \mathbb{R}^{S_q \times S_{kv}}$ multiply value matrix $V_h \in \mathbb{R}^{S_{kv} \times d_h}$ to produce attention output $\in \mathbb{R}^{S_q \times d_h}$ per head.

\begin{table}[h]
\centering
\caption{Score-value multiplication costs for weighted value aggregation.}
\label{tab:score_v_matmul}
\small
\begin{tabular}{lcc}
\toprule
\textbf{Metric} & \textbf{Prefill} ($S_q = S$) & \textbf{Decode} ($S_q = 1$) \\
\midrule
FLOPs & $2BS^2 d$ & $2BSd$ \\
Load Scores & $B n_h S^2 \cdot b_a$ & $B n_h S \cdot b_a$ \\
Load V Cache & $BSd_m \cdot b_{kv}$ & $BSd_m \cdot b_{kv}$ \\
Store Output & $BSd \cdot b_a$ & $Bd \cdot b_a$ \\
\bottomrule
\end{tabular}
\end{table}

Like QK matmul, score-value multiplication exhibits $O(S^2)$ complexity in prefill and $O(S)$ in decode.

\paragraph{Summary.}
Combining all three operations (QK matmul, softmax, score-value matmul):

\begin{table}[h]
\centering
\caption{Total computational and memory costs for attention score computation.}
\label{tab:attention_score_summary}
\small
\begin{tabular}{lcc}
\toprule
\textbf{Metric} & \textbf{Prefill} ($S_q = S$) & \textbf{Decode} ($S_q = 1$) \\
\midrule
Total FLOPs & $4BS^2 d + O(Bn_h S^2)$ & $4BSd + O(Bn_h S)$ \\
KV Cache Access & $2BSd_m \cdot b_{kv}$ & $2BSd_m \cdot b_{kv}$ \\
\bottomrule
\end{tabular}
\end{table}

\paragraph{Key Observations.}
From~\autoref{tab:attention_score_summary}:
\begin{itemize}
    \item \textbf{Quadratic vs. Linear Scaling:} Prefill attention scales as $O(S^2)$ due to full $S \times S$ attention matrix, while decode scales linearly as $O(S)$ since only one new token attends to all previous tokens.
    
    \item \textbf{Projection Dominance:} For typical configurations where $d \gg S$ (e.g., $d=4096$, $S=2048$), projection FLOPs ($O(Bd^2)$) dominate over attention score FLOPs ($O(BSd)$), especially in decode phase.
    
    \item \textbf{KV Cache Bottleneck:} KV cache access remains $O(BSd_m)$ in both phases, becoming a critical bottleneck in decode when sequence length $S$ is large.
\end{itemize}

\subsubsection{FFN Layers}

Modern transformers use either dense FFN layers or sparse Mixture-of-Experts (MoE) FFN layers.

\paragraph{Dense FFN with SwiGLU Activation.}
SwiGLU~\cite{shazeer2020glu} uses three linear projections with gated activation:
\begin{enumerate}
    \item \textbf{Gate projection}: $X \in \mathbb{R}^{B \times S \times d} \xrightarrow{W_g} G \in \mathbb{R}^{B \times S \times rd}$
    \item \textbf{Up projection}: $X \in \mathbb{R}^{B \times S \times d} \xrightarrow{W_u} U \in \mathbb{R}^{B \times S \times rd}$
    \item \textbf{Gated activation}: $H = \text{SiLU}(G) \odot U$ (element-wise)
    \item \textbf{Down projection}: $H \in \mathbb{R}^{B \times S \times rd} \xrightarrow{W_d} Y \in \mathbb{R}^{B \times S \times d}$
\end{enumerate}
where $r$ is the expansion ratio (typically $r = \frac{8}{3}$ or $r = 4$).

\begin{table}[h]
\centering
\caption{SwiGLU FFN costs. Element-wise ops contribute negligible $O(BSrd)$ FLOPs vs $O(BSrd^2)$ for projections.}
\label{tab:dense_ffn}
\small
\begin{tabular}{lcc}
\toprule
\textbf{Component} & \textbf{FLOPs (Prefill/Decode)} & \textbf{Weight Memory} \\
\midrule
Gate ($W_g \in \mathbb{R}^{d \times rd}$) & $2BSrd^2$ / $2Brd^2$ & $rd^2 \cdot b_w$ \\
Up ($W_u \in \mathbb{R}^{d \times rd}$) & $2BSrd^2$ / $2Brd^2$ & $rd^2 \cdot b_w$ \\
Down ($W_d \in \mathbb{R}^{rd \times d}$) & $2BSrd^2$ / $2Brd^2$ & $rd^2 \cdot b_w$ \\
\midrule
\textbf{Total (3 projections)} & $6BSrd^2$ / $6Brd^2$ & $3rd^2 \cdot b_w$ \\
\bottomrule
\end{tabular}
\end{table}

\paragraph{Mixture-of-Experts (MoE) FFN.}
MoE~\citep{shazeer2017outrageously,fedus2022switch} replaces dense FFN with $E$ parallel experts, activating only top-$K$ per token via learned routing, decoupling computation from capacity.

\textbf{Key parameters:}
$E$ = total experts; $K$ = active experts/token (typically $K=2$); $\rho = K/E$ = activation rate; $r_{\text{single}}$ = per-expert expansion; $r = K \cdot r_{\text{single}}$ = effective expansion.

\textbf{Costs:} FLOPs match dense FFN with same effective $r$, but all $E$ experts must be stored:
\begin{align}
\text{MoE FLOPs} &= 6BSrd^2 \text{ (prefill)}, \quad 6Brd^2 \text{ (decode)} \\
\text{MoE Weight Memory} &= \frac{3rd^2 b_w}{\rho} = E \cdot 3 r_{\text{single}} d^2 b_w
\end{align}

\begin{table}[h]
\centering
\caption{Dense vs MoE FFN comparison. MoE achieves $1/\rho$ more capacity at same FLOPs.}
\label{tab:ffn_comparison}
\small
\begin{tabular}{l|cc|c}
\toprule
& \multicolumn{2}{c|}{\textbf{FLOPs}} & \textbf{Weight Memory} \\
\textbf{Architecture} & Prefill & Decode & Total \\
\midrule
Dense FFN & $6BSrd^2$ & $6Brd^2$ & $3rd^2 b_w$ \\
MoE FFN & $6BSrd^2$ & $6Brd^2$ & $\dfrac{3rd^2 b_w}{\rho}$ \\
\midrule
\textbf{Ratio (MoE/Dense)} & $1\times$ & $1\times$ & $E/K$ \\
\bottomrule
\end{tabular}
\end{table}

\paragraph{Key Observations.}
\begin{itemize}
    \item \textbf{Capacity-Computation Decoupling:} MoE scales parameters without increasing FLOPs. Example: $E=8$, $K=2$ ($\rho=0.25$) yields $4\times$ more parameters at same compute.
    
    \item \textbf{Memory-Bound Regime:} In decode with small $B$, MoE's $1/\rho$ larger weight memory exacerbates bottlenecks, requiring quantization.
    
    \item \textbf{FFN Dominance:} FFN accounts for $\sim$60-70\% of total FLOPs when $r \approx 4$ (vs attention projection costs in~\autoref{eq:attn_proj_flops} and~\autoref{eq:attn_proj_weight}).
\end{itemize}

\subsection{Per-Layer Coefficient Summary}
\label{sec:coefficients}

Combining attention projections and FFN, the per-layer costs are:
\begin{align}
\text{FLOPs per layer} &= \underbrace{2BSd^2 \left(2 + \frac{2}{\mathrm{gqa}}\right)}_{\text{Attention Proj.}} + \underbrace{6BSrd^2}_{\text{FFN}} + \underbrace{O(BS^2 d)}_{\text{Attn Score}} \\
\text{Weight (compute)} &= \underbrace{d^2 b_w \left(2 + \frac{2}{\mathrm{gqa}}\right)}_{\text{Attention Proj.}} + \underbrace{3rd^2 b_w}_{\text{FFN (active)}} \\
\text{Weight (storage)} &= \underbrace{d^2 b_w \left(2 + \frac{2}{\mathrm{gqa}}\right)}_{\text{Attention Proj.}} + \underbrace{\frac{3rd^2 b_w}{\rho}}_{\text{FFN (all experts)}}
\end{align}

For large $d$ where projection FLOPs dominate, we define the normalized coefficients:
\begin{align}
\xi_F &= 4 + \frac{4}{\mathrm{gqa}} + 6r & \text{(FLOPs coefficient)} \label{eq:xi_F} \\
\xi_W^{\mathrm{dec}} &= 2 + \frac{2}{\mathrm{gqa}} + 3r & \text{(Decode weight coefficient)} \label{eq:xi_W_dec} \\
\xi_W^{\mathrm{all}} &= 2 + \frac{2}{\mathrm{gqa}} + \frac{3r}{\rho} & \text{(Storage coefficient)} \label{eq:xi_W_all}
\end{align}

\textbf{Key Relations}:
\begin{enumerate}[leftmargin=*, itemsep=2pt]
    \item $\xi_F = 2 \cdot \xi_W^{\mathrm{dec}}$ for any $(r, \mathrm{gqa})$. This factor of 2 arises because each multiply-accumulate operation counts as 2 FLOPs but requires loading each weight only once.
    \item $\rho$ appears only in $\xi_W^{\mathrm{all}}$, not in $\xi_F$ or $\xi_W^{\mathrm{dec}}$. This reflects that sparsity affects storage but not per-token computation.
    \item For dense models ($\rho = 1$): $\xi_W^{\mathrm{all}} = \xi_W^{\mathrm{dec}}$.
\end{enumerate}

\subsection{Inference Phases and Roofline Model}
\label{sec:inference_phases}

Autoregressive generation consists of two phases:

\paragraph{Prefill Phase.}
The model processes $S_{\mathrm{in}}$ input tokens in parallel. This phase is typically \textbf{compute-bound}.

\paragraph{Decode Phase.}
The model generates $S_{\mathrm{out}}$ tokens autoregressively. At step $t$, the KV-cache contains $S_{\mathrm{in}} + t$ entries. This phase is typically \textbf{memory-bandwidth-bound}.

\paragraph{Roofline Model.}
The roofline model~\citep{williams2009roofline} characterizes kernel latency as:
\begin{equation}
T = \max\left(\frac{\mathcal{F}}{\pi_H}, \frac{\mathcal{W}}{\beta_H}\right)
\end{equation}
where $\mathcal{F}$ is FLOPs, $\mathcal{W}$ is memory traffic, $\pi_H$ is peak compute, and $\beta_H$ is memory bandwidth.

\subsection{Latency Modeling}
\label{sec:latency_modeling_app}

\subsubsection{Prefill Latency}

For batch size $B$ and input sequence length $S_{\mathrm{in}}$, the total FLOPs per layer is:
\begin{equation}
\mathcal{F}_{\mathrm{layer}} = BS_{\mathrm{in}} d^2 \cdot \xi_F
\end{equation}

With large batch-sequence product $BS_{\mathrm{in}}$, prefill is typically compute-bound:
\begin{equation}
T_{\mathrm{pre}} = \frac{l \cdot BS_{\mathrm{in}} d^2 \cdot \xi_F}{\pi_H}
\label{eq:prefill_latency}
\end{equation}

\subsubsection{Decode Latency: Single Step Analysis}

At decode step $t$, processing $B$ tokens with context length $S_{\mathrm{in}} + t$:

\paragraph{Memory Traffic per Step.}

\textbf{(1) Weight loading:}
\begin{equation}
\mathcal{W}_{\mathrm{weight}} = \xi_W^{\mathrm{dec}} \cdot d^2 \cdot b_w
\end{equation}

\textbf{(2) KV-cache loading:}
\begin{equation}
\mathcal{W}_{\mathrm{KV}}(t) = \frac{2(S_{\mathrm{in}} + t) \cdot d \cdot b_{kv}}{\mathrm{gqa}}
\label{eq:kv_traffic}
\end{equation}

\textbf{Total per layer:}
\begin{equation}
\mathcal{W}_{\mathrm{layer}}(t) = \xi_W^{\mathrm{dec}} \cdot d^2 \cdot b_w + \frac{2(S_{\mathrm{in}} + t) \cdot d \cdot b_{kv}}{\mathrm{gqa}}
\end{equation}

\paragraph{Single Step Latency.}
With small batch $B$ (often $B = 1$), decode is typically memory-bound:
\begin{equation}
T_{\mathrm{step}}(t) = \frac{l}{\beta_H} \left[\xi_W^{\mathrm{dec}} d^2 b_w + \frac{2(S_{\mathrm{in}} + t) d b_{kv}}{\mathrm{gqa}}\right]
\end{equation}

\subsubsection{Decode Latency: Total Latency}

Summing over $t = 1, \ldots, S_{\mathrm{out}}$:
\begin{equation}
T_{\mathrm{decode}} = \frac{l}{\beta_H} \left[S_{\mathrm{out}} \cdot \xi_W^{\mathrm{dec}} d^2 b_w + \frac{2 d b_{kv}}{\mathrm{gqa}} \sum_{t=1}^{S_{\mathrm{out}}} (S_{\mathrm{in}} + t)\right]
\end{equation}

Evaluating the sum:
\begin{equation}
\sum_{t=1}^{S_{\mathrm{out}}} (S_{\mathrm{in}} + t) = S_{\mathrm{out}} \cdot S_{\mathrm{in}} + \frac{S_{\mathrm{out}}(S_{\mathrm{out}} + 1)}{2} = S_{\mathrm{out}} \cdot \bar{S}
\end{equation}
where the \textbf{average context length} is:
\begin{equation}
\bar{S} = S_{\mathrm{in}} + \frac{S_{\mathrm{out}} + 1}{2}
\label{eq:S_bar}
\end{equation}

\paragraph{Complete Decode Latency.}
\begin{equation}
T_{\mathrm{decode}} = \frac{l \cdot S_{\mathrm{out}}}{\beta_H} \left[\xi_W^{\mathrm{dec}} \cdot d^2 \cdot b_w + \frac{2\bar{S} \cdot d \cdot b_{kv}}{\mathrm{gqa}}\right]
\label{eq:decode_latency_full}
\end{equation}

\subsubsection{Decode Latency: Unified Form}

Define the \textbf{effective decode coefficient}:
\begin{equation}
\xi_W^{\mathrm{eff}}(d, \mathrm{gqa}) = \xi_W^{\mathrm{dec}} + \frac{2\bar{S} \cdot b_{kv}}{\mathrm{gqa} \cdot d \cdot b_w}
\label{eq:xi_W_eff}
\end{equation}

Then:
\begin{equation}
T_{\mathrm{decode}} = \frac{l \cdot S_{\mathrm{out}} \cdot \xi_W^{\mathrm{eff}} \cdot d^2 \cdot b_w}{\beta_H}
\label{eq:decode_latency_unified}
\end{equation}

Expanding $\xi_W^{\mathrm{eff}}$:
\begin{equation}
\xi_W^{\mathrm{eff}} = 2 + \frac{2}{\mathrm{gqa}} + 3r + \frac{2\bar{S} \cdot b_{kv}}{\mathrm{gqa} \cdot d \cdot b_w}
\end{equation}

\textbf{Note}: $\xi_W^{\mathrm{eff}}$ depends on $d$ and $\mathrm{gqa}$, which introduces additional coupling in optimization problems.

\subsection{Memory Footprint}
\label{sec:memory_footprint}

Per-layer storage (all $E$ experts):
\begin{equation}
M_{\mathrm{layer}} = \xi_W^{\mathrm{all}} \cdot d^2 \cdot b_w
\end{equation}

Total model memory:
\begin{equation}
M = l \cdot \xi_W^{\mathrm{all}} \cdot d^2 \cdot b_w
\label{eq:memory_footprint}
\end{equation}

\subsection{Constraint Formulations}
\label{sec:constraints}

\subsubsection{Prefill Constraint}

Given target latency $T_{\mathrm{lat}}^{\mathrm{pre}}$, define $\bar{F}_p = T_{\mathrm{lat}}^{\mathrm{pre}} \cdot \pi_H / (BS_{\mathrm{in}})$:
\begin{equation}
l \cdot \xi_F \cdot d^2 \leq \bar{F}_p
\label{eq:prefill_constraint}
\end{equation}

\subsubsection{Decode Constraint}

Given target latency $T_{\mathrm{lat}}^{\mathrm{dec}}$, define $\bar{M}_d = T_{\mathrm{lat}}^{\mathrm{dec}} \cdot \beta_H / S_{\mathrm{out}}$:
\begin{equation}
l \cdot \xi_W^{\mathrm{eff}} \cdot d^2 \cdot b_w \leq \bar{M}_d
\label{eq:decode_constraint}
\end{equation}

Expanding with the full KV-cache term:
\begin{equation}
l \cdot \xi_W^{\mathrm{dec}} \cdot d^2 \cdot b_w + \frac{2l \cdot \bar{S} \cdot d \cdot b_{kv}}{\mathrm{gqa}} \leq \bar{M}_d
\label{eq:decode_constraint_expanded}
\end{equation}

\subsubsection{Memory Constraint}

\begin{equation}
l \cdot \xi_W^{\mathrm{all}} \cdot d^2 \cdot b_w \leq M_{\mathrm{budget}}
\label{eq:memory_constraint}
\end{equation}

\subsection{Summary}
\label{sec:summary}

\begin{table}[h]
\caption{Roofline coefficients and their usage.}
\label{tab:coefficients}
\centering
\begin{tabular}{c|c|c}
\toprule
\textbf{Coefficient} & \textbf{Formula} & \textbf{Usage} \\
\midrule
$\xi_F$ & $4 + \frac{4}{\mathrm{gqa}} + 6r$ & Prefill FLOPs: $\mathcal{F} = l \cdot \xi_F \cdot BS \cdot d^2$ \\[6pt]
$\xi_W^{\mathrm{dec}}$ & $2 + \frac{2}{\mathrm{gqa}} + 3r$ & Decode weight traffic \\[6pt]
$\xi_W^{\mathrm{eff}}$ & $\xi_W^{\mathrm{dec}} + \frac{2\bar{S} b_{kv}}{\mathrm{gqa} \cdot d \cdot b_w}$ & Decode total traffic (weight + KV) \\[6pt]
$\xi_W^{\mathrm{all}}$ & $2 + \frac{2}{\mathrm{gqa}} + \frac{3r}{\rho}$ & Storage: $M = l \cdot \xi_W^{\mathrm{all}} \cdot d^2 \cdot b_w$ \\
\bottomrule
\end{tabular}
\end{table}

\begin{table}[h]
\caption{Partial derivatives of coefficients.}
\label{tab:derivatives}
\centering
\begin{tabular}{l|cccc}
\toprule
& $\partial/\partial r$ & $\partial/\partial \mathrm{gqa}$ & $\partial/\partial \rho$ & $\partial/\partial d$ \\
\midrule
$\xi_F$ & $6$ & $-\frac{4}{\mathrm{gqa}^2}$ & $0$ & $0$ \\[4pt]
$\xi_W^{\mathrm{dec}}$ & $3$ & $-\frac{2}{\mathrm{gqa}^2}$ & $0$ & $0$ \\[4pt]
$\xi_W^{\mathrm{eff}}$ & $3$ & $-\frac{2}{\mathrm{gqa}^2} - \frac{2\bar{S} b_{kv}}{\mathrm{gqa}^2 d b_w}$ & $0$ & $-\frac{2\bar{S} b_{kv}}{\mathrm{gqa} d^2 b_w}$ \\[4pt]
$\xi_W^{\mathrm{all}}$ & $\frac{3}{\rho}$ & $-\frac{2}{\mathrm{gqa}^2}$ & $-\frac{3r}{\rho^2}$ & $0$ \\
\bottomrule
\end{tabular}
\end{table}

\textbf{Note}: $\xi_W^{\mathrm{eff}}$ depends on $d$, which introduces additional coupling in the optimization. For dense models ($\rho = 1$), we have $\xi_W^{\mathrm{all}} = \xi_W^{\mathrm{dec}}$.

Additionally, we define the \textbf{aggregate loss gradient} $\tilde{D} \triangleq \kappa_\rho \rho^{\alpha_\rho} d^{\beta_2 - \beta_1} + \kappa_d$, which combines the sparsity and base capacity contributions from the loss function. This shorthand appears throughout the following case derivations (Cases D1--D3, P1--P3) in the stationarity conditions for $r$.

\section{Case D1: Decode, Latency-Constrained}
\label{app:D1}

\subsection{Problem Statement}

\begin{equation}
\begin{aligned}
\min_{l, d, r, \mathrm{gqa}, \rho} \quad & \hat{\mathcal{L}}(\theta) = \frac{\kappa_l}{l^{\alpha_l}} + \frac{\kappa_\rho \rho^{\alpha_\rho}}{r^{\alpha_r} d^{\beta_1}} + \frac{\kappa_d}{r^{\alpha_r} d^{\beta_2}} + \frac{\kappa_m \cdot \mathrm{gqa}^{\alpha_m}}{d^{\alpha_m}} + \hat{\mathcal{L}}_\infty \\
\mathrm{s.t.} \quad & l \cdot \xi_W^{\mathrm{dec}} \cdot d^2 \cdot b_w + \frac{2l \cdot \bar{S} \cdot d \cdot b_{kv}}{\mathrm{gqa}} = \bar{M}_d \\
& \rho \geq \rho_{\min}
\end{aligned}
\end{equation}

where $\xi_W^{\mathrm{dec}} = 2 + 2/\mathrm{gqa} + 3r$.

Define the constraint function:
\begin{equation}
g_T(\theta) = l \cdot \xi_W^{\mathrm{dec}} \cdot d^2 \cdot b_w + \frac{2l \cdot \bar{S} \cdot d \cdot b_{kv}}{\mathrm{gqa}} - \bar{M}_d
\end{equation}

\subsection{Lagrangian}

\begin{equation}
\mathscr{L} = \hat{\mathcal{L}}(\theta) + \mu_T \cdot g_T(\theta)
\end{equation}

\subsection{Constraint Partial Derivatives}

\begin{align}
\frac{\partial g_T}{\partial l} &= \xi_W^{\mathrm{dec}} \cdot d^2 \cdot b_w + \frac{2\bar{S} \cdot d \cdot b_{kv}}{\mathrm{gqa}} \\[6pt]
\frac{\partial g_T}{\partial d} &= 2l \cdot \xi_W^{\mathrm{dec}} \cdot d \cdot b_w + \frac{2l \cdot \bar{S} \cdot b_{kv}}{\mathrm{gqa}} \\[6pt]
\frac{\partial g_T}{\partial r} &= 3l \cdot d^2 \cdot b_w \\[6pt]
\frac{\partial g_T}{\partial \mathrm{gqa}} &= -\frac{2l \cdot d^2 \cdot b_w}{\mathrm{gqa}^2} - \frac{2l \cdot \bar{S} \cdot d \cdot b_{kv}}{\mathrm{gqa}^2} \\[6pt]
\frac{\partial g_T}{\partial \rho} &= 0
\end{align}

Note: $\partial g_T / \partial \rho = 0$ still holds because the KV-cache term is also independent of $\rho$.

\subsection{KKT Conditions}

\paragraph{Stationarity for $l$:}
\begin{equation}
-\frac{\alpha_l \kappa_l}{l^{\alpha_l + 1}} + \mu_T \left(\xi_W^{\mathrm{dec}} d^2 b_w + \frac{2\bar{S} d b_{kv}}{\mathrm{gqa}}\right) = 0
\end{equation}

Solving for $\mu_T$:
\begin{equation}
\mu_T = \frac{\alpha_l \kappa_l}{l^{\alpha_l + 1} \left(\xi_W^{\mathrm{dec}} d^2 b_w + \frac{2\bar{S} d b_{kv}}{\mathrm{gqa}}\right)}
\end{equation}

\paragraph{Stationarity for $\rho$:}
\begin{equation}
\frac{\partial \mathscr{L}}{\partial \rho} = \frac{\alpha_\rho \kappa_\rho \rho^{\alpha_\rho - 1}}{r^{\alpha_r} d^{\beta_1}} + \mu_T \cdot 0 = \frac{\alpha_\rho \kappa_\rho \rho^{\alpha_\rho - 1}}{r^{\alpha_r} d^{\beta_1}} > 0
\end{equation}

Since $\partial \mathscr{L}/\partial \rho > 0$ for all $\rho > 0$:
\begin{equation}
\rho^* = \rho_{\min}
\end{equation}

\textbf{Physical interpretation.} The activation rate $\rho$ does not appear in the decode latency constraint ($\partial g_T/\partial \rho = 0$), because only $K$ experts are activated per token regardless of the total pool size $E$---thus per-token computation and bandwidth cost are invariant to $\rho$. Meanwhile, the loss is monotonically increasing in $\rho$: fewer activated experts relative to total experts means greater total model capacity at no additional per-token cost. Therefore, the optimal strategy under latency constraints is to maximize sparsity (minimize $\rho$), i.e., increase the expert pool $E$ as far as memory permits while keeping $K$ fixed.

\paragraph{Stationarity for $r$:}
\begin{equation}
-\frac{\alpha_r \tilde{D}}{r^{\alpha_r + 1} d^{\beta_2}} + 3\mu_T l d^2 b_w = 0
\end{equation}

\paragraph{Stationarity for $\mathrm{gqa}$:}
\begin{equation}
\frac{\alpha_m \kappa_m \cdot \mathrm{gqa}^{\alpha_m - 1}}{d^{\alpha_m}} - \mu_T \left(\frac{2l d^2 b_w}{\mathrm{gqa}^2} + \frac{2l \bar{S} d b_{kv}}{\mathrm{gqa}^2}\right) = 0
\end{equation}

Simplifying:
\begin{equation}
\frac{\alpha_m \kappa_m \cdot \mathrm{gqa}^{\alpha_m - 1}}{d^{\alpha_m}} = \frac{2\mu_T l}{\mathrm{gqa}^2} \left(d^2 b_w + \bar{S} d b_{kv}\right)
\end{equation}

\paragraph{Stationarity for $d$:}
\begin{equation}
-\frac{\beta_1 \kappa_\rho \rho^{\alpha_\rho}}{r^{\alpha_r} d^{\beta_1 + 1}} - \frac{\beta_2 \kappa_d}{r^{\alpha_r} d^{\beta_2 + 1}} - \frac{\alpha_m \kappa_m \mathrm{gqa}^{\alpha_m}}{d^{\alpha_m + 1}} + \mu_T \left(2l \xi_W^{\mathrm{dec}} d b_w + \frac{2l \bar{S} b_{kv}}{\mathrm{gqa}}\right) = 0
\end{equation}

\subsection{Solution Derivation}

\paragraph{Step 1: Depth.}
From the active constraint:
\begin{equation}
l^* = \frac{\bar{M}_d}{\xi_W^{\mathrm{dec}} d^2 b_w + \frac{2\bar{S} d b_{kv}}{\mathrm{gqa}}}
\end{equation}

Or equivalently using $\xi_W^{\mathrm{eff}}$:
\begin{equation}
l^* = \frac{\bar{M}_d}{\xi_W^{\mathrm{eff}} \cdot d^2 \cdot b_w}
\end{equation}

\paragraph{Step 2: Multiplier.}
Define the effective constraint coefficient:
\begin{equation}
\Gamma \triangleq \xi_W^{\mathrm{dec}} d^2 b_w + \frac{2\bar{S} d b_{kv}}{\mathrm{gqa}} = \xi_W^{\mathrm{eff}} \cdot d^2 \cdot b_w
\end{equation}

Then:
\begin{equation}
\mu_T l = \frac{\alpha_l \kappa_l}{l^{\alpha_l} \Gamma} = \frac{\alpha_l \kappa_l \Gamma^{\alpha_l - 1}}{\bar{M}_d^{\alpha_l}}
\end{equation}

\paragraph{Step 3: FFN Ratio.}
From stationarity for $r$:
\begin{equation}
\frac{\alpha_r \tilde{D}}{r^{\alpha_r + 1} d^{\beta_2}} = 3\mu_T l d^2 b_w
\end{equation}

\begin{equation}
r^* = \left[\frac{\alpha_r \tilde{D}}{3\alpha_l \kappa_l} \cdot \frac{\bar{M}_d^{\alpha_l}}{\Gamma^{\alpha_l - 1} d^{2 + \beta_2} b_w}\right]^{\frac{1}{\alpha_r + 1}}
\end{equation}

\paragraph{Step 4: GQA Ratio.}
From stationarity for $\mathrm{gqa}$:
\begin{equation}
\mathrm{gqa}^{\alpha_m + 1} = \frac{2\mu_T l d^{\alpha_m}}{\alpha_m \kappa_m} \left(d^2 b_w + \bar{S} d b_{kv}\right)
\end{equation}

\begin{equation}
\mathrm{gqa}^* = \left[\frac{2\alpha_l \kappa_l}{\alpha_m \kappa_m} \cdot \frac{\Gamma^{\alpha_l - 1} d^{\alpha_m}}{\bar{M}_d^{\alpha_l}} \left(d^2 b_w + \bar{S} d b_{kv}\right)\right]^{\frac{1}{\alpha_m + 1}}
\end{equation}

\subsection{Solution Summary}

\begin{align}
\rho^* &= \rho_{\min} \\[4pt]
l^* &= \frac{\bar{M}_d}{\xi_W^{\mathrm{dec}} d^2 b_w + \frac{2\bar{S} d b_{kv}}{\mathrm{gqa}}} \\[4pt]
r^* &= \left[\frac{\alpha_r \tilde{D}}{3\alpha_l \kappa_l} \cdot \frac{\bar{M}_d^{\alpha_l}}{\Gamma^{\alpha_l - 1} d^{2 + \beta_2} b_w}\right]^{\frac{1}{\alpha_r + 1}} \\[4pt]
\mathrm{gqa}^* &= \left[\frac{2\alpha_l \kappa_l}{\alpha_m \kappa_m} \cdot \frac{\Gamma^{\alpha_l - 1} d^{\alpha_m}}{\bar{M}_d^{\alpha_l}} \left(d^2 b_w + \bar{S} d b_{kv}\right)\right]^{\frac{1}{\alpha_m + 1}}
\end{align}

where $\Gamma = \xi_W^{\mathrm{dec}} d^2 b_w + 2\bar{S} d b_{kv} / \mathrm{gqa}$ and $\xi_W^{\mathrm{dec}} = 2 + 2/\mathrm{gqa}^* + 3r^*$ (implicit). The activation rate result corresponds to \autoref{thm:latency-only} in the main text.


\section{Case D2: Decode, Memory-Constrained}
\label{app:D2}

\subsection{Problem Statement}

\begin{equation}
\begin{aligned}
\min_{l, d, r, \mathrm{gqa}, \rho} \quad & \hat{\mathcal{L}}(\theta) \\
\mathrm{s.t.} \quad & l \cdot \xi_W^{\mathrm{all}} \cdot d^2 \cdot b_w = M_{\mathrm{budget}}
\end{aligned}
\end{equation}

where $\xi_W^{\mathrm{all}} = 2 + 2/\mathrm{gqa} + 3r/\rho$.

Note: The memory constraint concerns model storage, which is independent of the KV-cache runtime overhead.

\subsection{Lagrangian}

\begin{equation}
\mathscr{L} = \hat{\mathcal{L}}(\theta) + \mu_M \left(l \cdot \xi_W^{\mathrm{all}} \cdot d^2 \cdot b_w - M_{\mathrm{budget}}\right)
\end{equation}

\subsection{Constraint Partial Derivatives}

\begin{align}
\frac{\partial g_M}{\partial l} &= \xi_W^{\mathrm{all}} d^2 b_w \\[4pt]
\frac{\partial g_M}{\partial d} &= 2l \xi_W^{\mathrm{all}} d b_w \\[4pt]
\frac{\partial g_M}{\partial r} &= \frac{3l d^2 b_w}{\rho} \\[4pt]
\frac{\partial g_M}{\partial \mathrm{gqa}} &= -\frac{2l d^2 b_w}{\mathrm{gqa}^2} \\[4pt]
\frac{\partial g_M}{\partial \rho} &= -\frac{3lr d^2 b_w}{\rho^2}
\end{align}

\subsection{KKT Conditions}

\paragraph{Stationarity for $l$:}
\begin{equation}
-\frac{\alpha_l \kappa_l}{l^{\alpha_l + 1}} + \mu_M \xi_W^{\mathrm{all}} d^2 b_w = 0
\quad \Rightarrow \quad
\mu_M = \frac{\alpha_l \kappa_l}{l^{\alpha_l + 1} \xi_W^{\mathrm{all}} d^2 b_w}
\end{equation}

\paragraph{Stationarity for $\rho$:}
\begin{equation}
\frac{\alpha_\rho \kappa_\rho \rho^{\alpha_\rho - 1}}{r^{\alpha_r} d^{\beta_1}} - \frac{3\mu_M lr d^2 b_w}{\rho^2} = 0
\end{equation}

Rearranging:
\begin{equation}
\mu_M l = \frac{\alpha_\rho \kappa_\rho \rho^{\alpha_\rho + 1}}{3r^{\alpha_r + 1} d^{\beta_1 + 2} b_w}
\label{eq:D2_muM_l_rho}
\end{equation}

\paragraph{Stationarity for $r$:}
\begin{equation}
-\frac{\alpha_r \tilde{D}}{r^{\alpha_r + 1} d^{\beta_2}} + \frac{3\mu_M l d^2 b_w}{\rho} = 0
\end{equation}

Rearranging:
\begin{equation}
\mu_M l = \frac{\alpha_r \tilde{D} \rho}{3 r^{\alpha_r + 1} d^{\beta_2 + 2} b_w}
\label{eq:D2_muM_l_r}
\end{equation}

\paragraph{Stationarity for $\mathrm{gqa}$:}
\begin{equation}
\frac{\alpha_m \kappa_m \cdot \mathrm{gqa}^{\alpha_m - 1}}{d^{\alpha_m}} - \frac{2\mu_M l d^2 b_w}{\mathrm{gqa}^2} = 0
\end{equation}

\subsection{Key Derivation: Activation Rate}

\autoref{eq:D2_muM_l_rho} and~\autoref{eq:D2_muM_l_r}:
\begin{equation}
\frac{\alpha_\rho \kappa_\rho \rho^{\alpha_\rho + 1}}{3r^{\alpha_r + 1} d^{\beta_1 + 2} b_w} = \frac{\alpha_r \tilde{D} \rho}{3 r^{\alpha_r + 1} d^{\beta_2 + 2} b_w}
\end{equation}

Canceling common factors:
\begin{equation}
\frac{\alpha_\rho \kappa_\rho \rho^{\alpha_\rho}}{d^{\beta_1}} = \frac{\alpha_r \tilde{D}}{d^{\beta_2}}
\end{equation}

Substituting $\tilde{D} = \kappa_\rho \rho^{\alpha_\rho} d^{\beta_2 - \beta_1} + \kappa_d$:
\begin{equation}
\alpha_\rho \kappa_\rho \rho^{\alpha_\rho} d^{\beta_2 - \beta_1} = \alpha_r \kappa_\rho \rho^{\alpha_\rho} d^{\beta_2 - \beta_1} + \alpha_r \kappa_d
\end{equation}

Collecting terms:
\begin{equation}
(\alpha_\rho - \alpha_r) \kappa_\rho \rho^{\alpha_\rho} d^{\beta_2 - \beta_1} = \alpha_r \kappa_d
\end{equation}

Solving:
\begin{equation}
\rho^* = \left[\frac{\alpha_r \kappa_d}{(\alpha_\rho - \alpha_r) \kappa_\rho}\right]^{1/\alpha_\rho} d^{(\beta_1 - \beta_2)/\alpha_\rho}
\end{equation}

Validity requires $\alpha_\rho > \alpha_r$.

\subsection{Remaining Solutions}

\paragraph{Depth.}
\begin{equation}
l^* = \frac{M_{\mathrm{budget}}}{\xi_W^{\mathrm{all}} d^2 b_w}
\end{equation}

\paragraph{Multiplier.}
\begin{equation}
\mu_M l = \frac{\alpha_l \kappa_l (\xi_W^{\mathrm{all}})^{\alpha_l - 1} d^{2(\alpha_l - 1)} b_w^{\alpha_l - 1}}{M_{\mathrm{budget}}^{\alpha_l}}
\end{equation}

\paragraph{FFN Ratio.}
\begin{equation}
r^* = \left[\frac{\alpha_r \tilde{D} \rho^*}{3\alpha_l \kappa_l} \cdot \frac{M_{\mathrm{budget}}^{\alpha_l}}{(\xi_W^{\mathrm{all}})^{\alpha_l - 1} d^{2\alpha_l + \beta_2} b_w^{\alpha_l}}\right]^{\frac{1}{\alpha_r + 1}}
\end{equation}

\paragraph{GQA Ratio.}
\begin{equation}
\mathrm{gqa}^* = \left[\frac{2\alpha_l \kappa_l}{\alpha_m \kappa_m} \cdot \frac{(\xi_W^{\mathrm{all}})^{\alpha_l - 1} d^{2\alpha_l + \alpha_m} b_w^{\alpha_l}}{M_{\mathrm{budget}}^{\alpha_l}}\right]^{\frac{1}{\alpha_m + 1}}
\end{equation}

\subsection{Solution Summary}

\begin{align}
\rho^* &= \left[\frac{\alpha_r \kappa_d}{(\alpha_\rho - \alpha_r) \kappa_\rho}\right]^{1/\alpha_\rho} d^{(\beta_1 - \beta_2)/\alpha_\rho} \\[4pt]
l^* &= \frac{M_{\mathrm{budget}}}{\xi_W^{\mathrm{all}} d^2 b_w} \\[4pt]
r^* &= \left[\frac{\alpha_r \tilde{D} \rho^*}{3\alpha_l \kappa_l} \cdot \frac{M_{\mathrm{budget}}^{\alpha_l}}{(\xi_W^{\mathrm{all}})^{\alpha_l - 1} d^{2\alpha_l + \beta_2} b_w^{\alpha_l}}\right]^{\frac{1}{\alpha_r + 1}} \\[4pt]
\mathrm{gqa}^* &= \left[\frac{2\alpha_l \kappa_l}{\alpha_m \kappa_m} \cdot \frac{(\xi_W^{\mathrm{all}})^{\alpha_l - 1} d^{2\alpha_l + \alpha_m} b_w^{\alpha_l}}{M_{\mathrm{budget}}^{\alpha_l}}\right]^{\frac{1}{\alpha_m + 1}}
\end{align}

where $\xi_W^{\mathrm{all}} = 2 + 2/\mathrm{gqa}^* + 3r^*/\rho^*$ (implicit). The activation rate result corresponds to \autoref{thm:memory-only} in the main text.


\section{Case D3: Decode, Dual-Constrained}
\label{app:D3}

\subsection{Problem Statement}

\begin{equation}
\begin{aligned}
\min_{l, d, r, \mathrm{gqa}, \rho} \quad & \hat{\mathcal{L}}(\theta) \\
\mathrm{s.t.} \quad & l \cdot \xi_W^{\mathrm{dec}} \cdot d^2 \cdot b_w + \frac{2l \cdot \bar{S} \cdot d \cdot b_{kv}}{\mathrm{gqa}} = \bar{M}_d \\
& l \cdot \xi_W^{\mathrm{all}} \cdot d^2 \cdot b_w = M_{\mathrm{budget}}
\end{aligned}
\end{equation}

\subsection{Constraint Compatibility}

Define:
\begin{equation}
\Gamma \triangleq \xi_W^{\mathrm{dec}} \cdot d^2 \cdot b_w + \frac{2\bar{S} \cdot d \cdot b_{kv}}{\mathrm{gqa}}
\end{equation}

From the two constraints:
\begin{equation}
l \cdot \Gamma = \bar{M}_d, \quad l \cdot \xi_W^{\mathrm{all}} \cdot d^2 \cdot b_w = M_{\mathrm{budget}}
\end{equation}

Dividing:
\begin{equation}
\frac{\Gamma}{\xi_W^{\mathrm{all}} \cdot d^2 \cdot b_w} = \frac{\bar{M}_d}{M_{\mathrm{budget}}} \triangleq \eta
\end{equation}

Expanding:
\begin{equation}
\frac{\xi_W^{\mathrm{dec}} \cdot d^2 \cdot b_w + \frac{2\bar{S} \cdot d \cdot b_{kv}}{\mathrm{gqa}}}{\xi_W^{\mathrm{all}} \cdot d^2 \cdot b_w} = \eta
\end{equation}

\begin{equation}
\frac{\xi_W^{\mathrm{dec}}}{\xi_W^{\mathrm{all}}} + \frac{2\bar{S} \cdot b_{kv}}{\xi_W^{\mathrm{all}} \cdot \mathrm{gqa} \cdot d \cdot b_w} = \eta
\label{eq:D3_compatibility}
\end{equation}

\subsection{Derivation of Activation Rate}

Substituting $\xi_W^{\mathrm{dec}} = \alpha_{\mathrm{attn}} + 3r$ and $\xi_W^{\mathrm{all}} = \alpha_{\mathrm{attn}} + 3r/\rho$ where $\alpha_{\mathrm{attn}} = 2 + 2/\mathrm{gqa}$:

\begin{equation}
\frac{\alpha_{\mathrm{attn}} + 3r}{\alpha_{\mathrm{attn}} + \frac{3r}{\rho}} + \frac{2\bar{S} \cdot b_{kv}}{\left(\alpha_{\mathrm{attn}} + \frac{3r}{\rho}\right) \mathrm{gqa} \cdot d \cdot b_w} = \eta
\end{equation}

Define the KV-cache correction term:
\begin{equation}
\delta \triangleq \frac{2\bar{S} \cdot b_{kv}}{\mathrm{gqa} \cdot d \cdot b_w}
\end{equation}

Then:
\begin{equation}
\frac{\alpha_{\mathrm{attn}} + 3r + \frac{\delta}{\xi_W^{\mathrm{all}}}}{\xi_W^{\mathrm{all}}} = \eta
\end{equation}

This gives:
\begin{equation}
\alpha_{\mathrm{attn}} + 3r + \frac{\delta}{\xi_W^{\mathrm{all}}} = \eta \xi_W^{\mathrm{all}} = \eta \left(\alpha_{\mathrm{attn}} + \frac{3r}{\rho}\right)
\end{equation}

Rearranging:
\begin{equation}
\alpha_{\mathrm{attn}}(1 - \eta) + 3r + \frac{\delta}{\alpha_{\mathrm{attn}} + \frac{3r}{\rho}} = \frac{3\eta r}{\rho}
\end{equation}

This is an implicit equation for $\rho$ due to the presence of $\rho$ in the $\delta$ term's denominator.

\paragraph{Simplified Form.}
Multiplying through by $\xi_W^{\mathrm{all}} = \alpha_{\mathrm{attn}} + 3r/\rho$:
\begin{equation}
(\alpha_{\mathrm{attn}} + 3r)\xi_W^{\mathrm{all}} + \delta = \eta (\xi_W^{\mathrm{all}})^2
\end{equation}

Let $x = \xi_W^{\mathrm{all}}$:
\begin{equation}
(\alpha_{\mathrm{attn}} + 3r)x + \delta = \eta x^2
\end{equation}

\begin{equation}
\eta x^2 - (\alpha_{\mathrm{attn}} + 3r)x - \delta = 0
\end{equation}

Solving the quadratic:
\begin{equation}
\xi_W^{\mathrm{all}} = \frac{(\alpha_{\mathrm{attn}} + 3r) + \sqrt{(\alpha_{\mathrm{attn}} + 3r)^2 + 4\eta\delta}}{2\eta}
\end{equation}

From $\xi_W^{\mathrm{all}} = \alpha_{\mathrm{attn}} + 3r/\rho$:
\begin{equation}
\frac{3r}{\rho} = \xi_W^{\mathrm{all}} - \alpha_{\mathrm{attn}}
\end{equation}

\begin{equation}
\rho^* = \frac{3r}{\xi_W^{\mathrm{all}} - \alpha_{\mathrm{attn}}} = \frac{6\eta r}{(\alpha_{\mathrm{attn}} + 3r) - 2\eta\alpha_{\mathrm{attn}} + \sqrt{(\alpha_{\mathrm{attn}} + 3r)^2 + 4\eta\delta}}
\end{equation}

where $\delta = 2\bar{S} b_{kv} / (\mathrm{gqa} \cdot d \cdot b_w)$ and $\alpha_{\mathrm{attn}} = 2 + 2/\mathrm{gqa}$.

\subsection{Special Case: $\delta \to 0$}

When the KV-cache term is negligible ($\delta \to 0$):
\begin{equation}
\xi_W^{\mathrm{all}} \to \frac{\alpha_{\mathrm{attn}} + 3r}{\eta}
\end{equation}

\begin{equation}
\rho^* \to \frac{3\eta r}{\alpha_{\mathrm{attn}} + 3r - \eta\alpha_{\mathrm{attn}}} = \frac{3\eta r}{\alpha_{\mathrm{attn}}(1-\eta) + 3r}
\end{equation}

This recovers the simplified formula from the weight-dominated approximation.

\subsection{Remaining Solutions}

\paragraph{Depth.}
\begin{equation}
l^* = \frac{M_{\mathrm{budget}}}{\xi_W^{\mathrm{all}} d^2 b_w}
\end{equation}

\paragraph{Other Variables.}
Unlike the single-constraint cases (D1, D2, P1), where a single Lagrange multiplier allows sequential elimination, the dual-constrained case involves two active constraints with multipliers $\mu_T$ and $\mu_M$. The stationarity conditions for $r$ and $\mathrm{gqa}$ each depend on both multipliers, and the KV-cache term in the decode constraint further couples $\mathrm{gqa}$ to the latency budget. As a result, $r^*$, $\mathrm{gqa}^*$, and $d^*$ do not admit independent closed-form expressions and must be obtained by numerically solving the coupled KKT system. In practice, $\rho^*$ and $l^*$ from the equations above are substituted first, reducing the system to three unknowns.

\subsection{Solution Summary}

\begin{align}
\eta &= \bar{M}_d / M_{\mathrm{budget}} \\[4pt]
\delta &= \frac{2\bar{S} b_{kv}}{\mathrm{gqa} \cdot d \cdot b_w} \\[4pt]
\xi_W^{\mathrm{all}} &= \frac{(\alpha_{\mathrm{attn}} + 3r) + \sqrt{(\alpha_{\mathrm{attn}} + 3r)^2 + 4\eta\delta}}{2\eta} \\[4pt]
\rho^* &= \frac{3r}{\xi_W^{\mathrm{all}} - \alpha_{\mathrm{attn}}} \\[4pt]
l^* &= \frac{M_{\mathrm{budget}}}{\xi_W^{\mathrm{all}} d^2 b_w}
\end{align}

where $\alpha_{\mathrm{attn}} = 2 + 2/\mathrm{gqa}^*$, and $r^*$, $\mathrm{gqa}^*$ are coupled solutions. The activation rate result corresponds to \autoref{thm:dual-constrained}(b) in the main text.


\section{Case P1: Prefill, Latency-Constrained}
\label{app:P1}

\subsection{Problem Statement}

\begin{equation}
\begin{aligned}
\min_{l, d, r, \mathrm{gqa}, \rho} \quad & \hat{\mathcal{L}}(\theta) \\
\mathrm{s.t.} \quad & l \cdot \xi_F \cdot d^2 = \bar{F}_p \\
& \rho \geq \rho_{\min}
\end{aligned}
\end{equation}

where $\xi_F = 4 + 4/\mathrm{gqa} + 6r$.

Note: The prefill constraint is compute-bound and does not involve KV-cache traffic.

\subsection{Lagrangian}

\begin{equation}
\mathscr{L} = \hat{\mathcal{L}}(\theta) + \mu_T \left(l \cdot \xi_F \cdot d^2 - \bar{F}_p\right)
\end{equation}

\subsection{Constraint Partial Derivatives}

\begin{align}
\frac{\partial g_T}{\partial l} &= \xi_F d^2 \\[4pt]
\frac{\partial g_T}{\partial d} &= 2l \xi_F d \\[4pt]
\frac{\partial g_T}{\partial r} &= 6l d^2 \\[4pt]
\frac{\partial g_T}{\partial \mathrm{gqa}} &= -\frac{4l d^2}{\mathrm{gqa}^2} \\[4pt]
\frac{\partial g_T}{\partial \rho} &= 0
\end{align}

\subsection{KKT Conditions}

\paragraph{Stationarity for $l$:}
\begin{equation}
-\frac{\alpha_l \kappa_l}{l^{\alpha_l + 1}} + \mu_T \xi_F d^2 = 0
\quad \Rightarrow \quad
\mu_T = \frac{\alpha_l \kappa_l}{l^{\alpha_l + 1} \xi_F d^2}
\end{equation}

\paragraph{Stationarity for $\rho$:}
\begin{equation}
\frac{\partial \mathscr{L}}{\partial \rho} = \frac{\alpha_\rho \kappa_\rho \rho^{\alpha_\rho - 1}}{r^{\alpha_r} d^{\beta_1}} > 0
\end{equation}

Since $\partial g_T / \partial \rho = 0$ and $\partial \hat{\mathcal{L}}/\partial \rho > 0$:
\begin{equation}
\rho^* = \rho_{\min}
\end{equation}

\paragraph{Stationarity for $r$:}
\begin{equation}
-\frac{\alpha_r \tilde{D}}{r^{\alpha_r + 1} d^{\beta_2}} + 6\mu_T l d^2 = 0
\end{equation}

Note: coefficient is $6$ (vs $3$ in D1) from $\partial \xi_F / \partial r = 6$. The $b_w$ factor is absent because the prefill constraint operates in FLOPs rather than bytes, yielding $\partial g_T/\partial r = 6ld^2$ without the byte-width scaling present in the decode case ($\partial g_T/\partial r = 3ld^2 b_w$).

\paragraph{Stationarity for $\mathrm{gqa}$:}
\begin{equation}
\frac{\alpha_m \kappa_m \cdot \mathrm{gqa}^{\alpha_m - 1}}{d^{\alpha_m}} - \frac{4\mu_T l d^2}{\mathrm{gqa}^2} = 0
\end{equation}

Note: coefficient is $4$ (vs $2$ in D1) from $|\partial \xi_F / \partial \mathrm{gqa}| = 4/\mathrm{gqa}^2$.

\subsection{Solution Derivation}

\paragraph{Depth.}
\begin{equation}
l^* = \frac{\bar{F}_p}{\xi_F d^2}
\end{equation}

\paragraph{Multiplier.}
\begin{equation}
\mu_T l = \frac{\alpha_l \kappa_l \xi_F^{\alpha_l - 1} d^{2(\alpha_l - 1)}}{\bar{F}_p^{\alpha_l}}
\end{equation}

\paragraph{FFN Ratio.}
\begin{equation}
r^* = \left[\frac{\alpha_r \tilde{D}}{6\alpha_l \kappa_l} \cdot \frac{\bar{F}_p^{\alpha_l}}{\xi_F^{\alpha_l - 1} d^{2\alpha_l + \beta_2}}\right]^{\frac{1}{\alpha_r + 1}}
\end{equation}

\paragraph{GQA Ratio.}
\begin{equation}
\mathrm{gqa}^* = \left[\frac{4\alpha_l \kappa_l}{\alpha_m \kappa_m} \cdot \frac{\xi_F^{\alpha_l - 1} d^{2\alpha_l + \alpha_m}}{\bar{F}_p^{\alpha_l}}\right]^{\frac{1}{\alpha_m + 1}}
\end{equation}

\subsection{Comparison with Case D1}

\begin{table}[h]
\centering
\caption{Comparison of prefill (P1) and decode (D1) phase optimal hyperparameters. Decode phase requires larger expansion ratio ($r^*$ coefficient doubles) and smaller GQA groups ($\text{gqa}^*$ coefficient halves) due to memory-bound constraints.}
\label{tab:p1_d1_comparison}
\begin{tabular}{lcc}
\toprule
& \textbf{P1 (Prefill)} & \textbf{D1 (Decode)} \\
\midrule
$r^*$ coefficient & $1/6$ & $1/3$ \\
$\text{gqa}^*$ coefficient & $4$ & $2$ \\
Constraint & $\xi_F$ & $\xi_W^{\text{dec}} + \text{KV term}$ \\
\bottomrule
\end{tabular}
\end{table}

\autoref{tab:p1_d1_comparison} reveals how optimal architectural choices differ between prefill and decode phases. In prefill (P1), the compute-bound regime ($\xi_F$ constraint) favors smaller expansion ratios and larger GQA groups. In decode (D1), memory bandwidth constraints ($\xi_W^{\text{dec}}$ plus KV-cache access) necessitate doubling the expansion ratio coefficient (from $1/6$ to $1/3$) while halving the GQA coefficient (from $4$ to $2$), reflecting the need to balance computation against memory access costs.

\subsection{Solution Summary}

\begin{align}
\rho^* &= \rho_{\min} \\[4pt]
l^* &= \frac{\bar{F}_p}{\xi_F d^2} \\[4pt]
r^* &= \left[\frac{\alpha_r \tilde{D}}{6\alpha_l \kappa_l} \cdot \frac{\bar{F}_p^{\alpha_l}}{\xi_F^{\alpha_l - 1} d^{2\alpha_l + \beta_2}}\right]^{\frac{1}{\alpha_r + 1}} \\[4pt]
\mathrm{gqa}^* &= \left[\frac{4\alpha_l \kappa_l}{\alpha_m \kappa_m} \cdot \frac{\xi_F^{\alpha_l - 1} d^{2\alpha_l + \alpha_m}}{\bar{F}_p^{\alpha_l}}\right]^{\frac{1}{\alpha_m + 1}}
\end{align}

where $\xi_F = 4 + 4/\mathrm{gqa}^* + 6r^*$ (implicit). The activation rate result corresponds to \autoref{thm:latency-only} in the main text.


\section{Case P2: Prefill, Memory-Constrained}
\label{app:P2}

\subsection{Problem Statement}

\begin{equation}
\begin{aligned}
\min_{l, d, r, \mathrm{gqa}, \rho} \quad & \hat{\mathcal{L}}(\theta) \\
\mathrm{s.t.} \quad & l \cdot \xi_W^{\mathrm{all}} \cdot d^2 \cdot b_w = M_{\mathrm{budget}}
\end{aligned}
\end{equation}

where $\xi_W^{\mathrm{all}} = 2 + 2/\mathrm{gqa} + 3r/\rho$.

\subsection{Equivalence to Case D2}

The Lagrangian is:
\begin{equation}
\mathscr{L} = \hat{\mathcal{L}}(\theta) + \mu_M \left(l \cdot \xi_W^{\mathrm{all}} \cdot d^2 \cdot b_w - M_{\mathrm{budget}}\right)
\end{equation}

This is identical to Case D2 because:
\begin{enumerate}
    \item The loss function $\hat{\mathcal{L}}(\theta)$ is independent of scenario
    \item The memory constraint concerns model storage, not runtime behavior
    \item $\rho$ appears only in $\xi_W^{\mathrm{all}}$
\end{enumerate}

All KKT conditions and solutions are identical to Case D2.

\subsection{Solution Summary}

\begin{align}
\rho^* &= \left[\frac{\alpha_r \kappa_d}{(\alpha_\rho - \alpha_r) \kappa_\rho}\right]^{1/\alpha_\rho} d^{(\beta_1 - \beta_2)/\alpha_\rho} \\[4pt]
l^* &= \frac{M_{\mathrm{budget}}}{\xi_W^{\mathrm{all}} d^2 b_w} \\[4pt]
r^* &= \left[\frac{\alpha_r \tilde{D} \rho^*}{3\alpha_l \kappa_l} \cdot \frac{M_{\mathrm{budget}}^{\alpha_l}}{(\xi_W^{\mathrm{all}})^{\alpha_l - 1} d^{2\alpha_l + \beta_2} b_w^{\alpha_l}}\right]^{\frac{1}{\alpha_r + 1}} \\[4pt]
\mathrm{gqa}^* &= \left[\frac{2\alpha_l \kappa_l}{\alpha_m \kappa_m} \cdot \frac{(\xi_W^{\mathrm{all}})^{\alpha_l - 1} d^{2\alpha_l + \alpha_m} b_w^{\alpha_l}}{M_{\mathrm{budget}}^{\alpha_l}}\right]^{\frac{1}{\alpha_m + 1}}
\end{align}

All solutions identical to Case D2. The activation rate result corresponds to \autoref{thm:memory-only} in the main text.


\section{Case P3: Prefill, Dual-Constrained}
\label{app:P3}

\subsection{Problem Statement}

\begin{equation}
\begin{aligned}
\min_{l, d, r, \mathrm{gqa}, \rho} \quad & \hat{\mathcal{L}}(\theta) \\
\mathrm{s.t.} \quad & l \cdot \xi_F \cdot d^2 = \bar{F}_p \\
& l \cdot \xi_W^{\mathrm{all}} \cdot d^2 \cdot b_w = M_{\mathrm{budget}}
\end{aligned}
\end{equation}

Note: The prefill constraint is compute-bound and does not involve KV-cache.

\subsection{Constraint Compatibility}

Dividing the two constraints:
\begin{equation}
\frac{\xi_F}{b_w \cdot \xi_W^{\mathrm{all}}} = \frac{\bar{F}_p}{M_{\mathrm{budget}}} \triangleq \eta_p
\end{equation}

\subsection{Derivation of Activation Rate}

Using $\xi_F = 2(\alpha_{\mathrm{attn}} + 3r)$ and $\xi_W^{\mathrm{all}} = \alpha_{\mathrm{attn}} + 3r/\rho$:
\begin{equation}
\frac{2(\alpha_{\mathrm{attn}} + 3r)}{b_w(\alpha_{\mathrm{attn}} + \frac{3r}{\rho})} = \eta_p
\end{equation}

Cross-multiplying:
\begin{equation}
2\alpha_{\mathrm{attn}} + 6r = \eta_p b_w \alpha_{\mathrm{attn}} + \frac{3\eta_p b_w r}{\rho}
\end{equation}

Rearranging:
\begin{equation}
\alpha_{\mathrm{attn}}(2 - \eta_p b_w) + 6r = \frac{3\eta_p b_w r}{\rho}
\end{equation}

Solving:
\begin{equation}
\rho^* = \frac{3\eta_p b_w r}{\alpha_{\mathrm{attn}}(2 - \eta_p b_w) + 6r}
\end{equation}

where $\eta_p = \bar{F}_p / M_{\mathrm{budget}}$ and $\alpha_{\mathrm{attn}} = 2 + 2/\mathrm{gqa}$.

Validity requires $\eta_p b_w < 2$. This result corresponds to \autoref{thm:dual-constrained}(a) in the main text.

\subsection{Comparison with Case D3}

\begin{table}[h]
\centering
\caption{Comparison of prefill (P3) and decode (D3) phase characteristics. Decode includes KV-cache memory access costs and requires solving a quadratic equation for optimal depth-width ratio $\rho^*$.}
\label{tab:prefill_decode_comparison}
\begin{tabular}{l|cc}
\toprule
& \textbf{P3 (Prefill)} & \textbf{D3 (Decode)} \\
\midrule
$\eta$ definition & $\bar{F}_p / M_{\mathrm{budget}}$ & $\bar{M}_d / M_{\mathrm{budget}}$ \\[4pt]
Latency constraint & $l \xi_F d^2$ & $l(\xi_W^{\mathrm{dec}} d^2 b_w + \frac{2\bar{S} d b_{kv}}{\mathrm{gqa}})$ \\[4pt]
KV-cache term & absent & present \\[4pt]
$\rho^*$ formula & closed-form & involves quadratic \\
\bottomrule
\end{tabular}
\end{table}

As shown in~\autoref{tab:prefill_decode_comparison}, the key difference between prefill and decode phases lies in the memory access patterns and their impact on the optimal architecture. In prefill (P3), the latency constraint depends only on compute ($\xi_F d^2$) and admits a closed-form solution for $\rho^*$. In decode (D3), the latency constraint includes both weight loading and KV-cache access, with the KV-cache term $\frac{2Sdb_{kv}}{\text{gqa}}$ becoming dominant for large sequence lengths. This additional complexity requires solving a quadratic equation to find the optimal depth-width ratio.

\subsection{Remaining Solutions}

\paragraph{Depth.}
\begin{equation}
l^* = \frac{\bar{F}_p}{\xi_F d^2}
\end{equation}

\paragraph{Other Variables.}
The solutions for $r^*$, $\mathrm{gqa}^*$, $d^*$ require solving the coupled KKT system.

\subsection{Solution Summary}

\begin{align}
\eta_p &= \bar{F}_p / M_{\mathrm{budget}}, \quad \eta_p b_w < 2 \\[4pt]
\rho^* &= \frac{3\eta_p b_w r^*}{\alpha_{\mathrm{attn}}(2 - \eta_p b_w) + 6r^*} \\[4pt]
l^* &= \frac{\bar{F}_p}{\xi_F d^2}
\end{align}

where $\alpha_{\mathrm{attn}} = 2 + 2/\mathrm{gqa}^*$, and $r^*$, $\mathrm{gqa}^*$ are coupled solutions. The activation rate result corresponds to \autoref{thm:dual-constrained}(a) in the main text.


\section{Summary}
\label{app:summary}

\subsection{Constraint Forms}

The constraint forms are shown in~\autoref{tab:constraint_forms}.

\begin{table}[h]
\caption{Constraint formulations.}
\centering
\begin{tabular}{l|c}
\toprule
\textbf{Constraint} & \textbf{Formula} \\
\midrule
Prefill latency & $l \cdot \xi_F \cdot d^2 \leq \bar{F}_p$ \\[6pt]
Decode latency & $l \cdot \xi_W^{\mathrm{dec}} \cdot d^2 \cdot b_w + \frac{2l \cdot \bar{S} \cdot d \cdot b_{kv}}{\mathrm{gqa}} \leq \bar{M}_d$ \\[6pt]
Memory & $l \cdot \xi_W^{\mathrm{all}} \cdot d^2 \cdot b_w \leq M_{\mathrm{budget}}$ \\
\bottomrule
\end{tabular}
\label{tab:constraint_forms}
\end{table}

\subsection{Solution Matrix: Activation Rate $\rho^*$}

\begin{table}[h]
\centering
\caption{Solution Matrix: Activation Rate $\rho^*$.}
\begin{tabular}{l|c|c|c}
\toprule
& \textbf{Latency} & \textbf{Memory} & \textbf{Dual} \\
\midrule
\textbf{Decode} & $\rho_{\min}$ & $\left[\frac{\alpha_r \kappa_d}{(\alpha_\rho - \alpha_r) \kappa_\rho}\right]^{1/\alpha_\rho} d^{\frac{\beta_1 - \beta_2}{\alpha_\rho}}$ & quadratic form \\[8pt]
\textbf{Prefill} & $\rho_{\min}$ & same as Decode & $\frac{3\eta_p b_w r}{\alpha_{\mathrm{attn}}(2-\eta_p b_w)+6r}$ \\
\bottomrule
\end{tabular}
\label{tab:solution_matrix_rho}
\end{table}

The activation rate $\rho^*$ in solution matrix are shown in~\autoref{tab:solution_matrix_rho}.

\subsection{Solution Matrix: Depth $l^*$}

The Depth $l^*$ in solution matrix are shown in~\autoref{tab:solution_matrix_depth}.

\begin{table}[h]
\centering
\caption{Solution Matrix: Depth $l^*$.}
\begin{tabular}{l|c|c}
\toprule
& \textbf{Latency / Dual} & \textbf{Memory} \\
\midrule
\textbf{Decode} & $\frac{\bar{M}_d}{\xi_W^{\mathrm{dec}} d^2 b_w + \frac{2\bar{S} d b_{kv}}{\mathrm{gqa}}}$ & $\frac{M_{\mathrm{budget}}}{\xi_W^{\mathrm{all}} d^2 b_w}$ \\[10pt]
\textbf{Prefill} & $\frac{\bar{F}_p}{\xi_F d^2}$ & $\frac{M_{\mathrm{budget}}}{\xi_W^{\mathrm{all}} d^2 b_w}$ \\
\bottomrule
\end{tabular}
\label{tab:solution_matrix_depth}
\end{table}

\subsection{Coefficient Comparison}

The coefficient comparison are shown in~\autoref{tab:coefficient_comparison}.

\begin{table}[H]
\caption{Coefficient Comparison.}
\centering
\begin{tabular}{l|cc|cc}
\toprule
& \multicolumn{2}{c|}{$r^*$ coefficient} & \multicolumn{2}{c}{$\mathrm{gqa}^*$ coefficient} \\
& Prefill & Decode & Prefill & Decode \\
\midrule
Latency & $1/6$ & $1/3$ & $4$ & $2$ \\
Memory & $1/3$ & $1/3$ & $2$ & $2$ \\
\bottomrule
\end{tabular}
\label{tab:coefficient_comparison}
\end{table}

\subsection{Key Results}

\paragraph{1. Width-Sparsity Scaling.}
Under memory constraint:
\begin{equation}
\rho^* = \left[\frac{\alpha_r \kappa_d}{(\alpha_\rho - \alpha_r) \kappa_\rho}\right]^{1/\alpha_\rho} d^{(\beta_1 - \beta_2)/\alpha_\rho}
\end{equation}

\paragraph{2. Scenario Independence.}
Memory-constrained $\rho^*$ is identical for Prefill and Decode.

\paragraph{3. Prefill-Decode Asymmetry.}
Latency-constrained coefficients differ since $\partial \xi_F/\partial r = 6$, comparing to $\partial \xi_W^{\mathrm{dec}}/\partial r = 3$.

\paragraph{4. KV-Cache Effect.}
The decode latency constraint includes the KV-cache term $2l\bar{S}db_{kv}/\mathrm{gqa}$, which:
\begin{itemize}
    \item Increases effective constraint tightness
    \item Couples $\mathrm{gqa}$ more strongly to the constraint
    \item Leads to quadratic form in D3 dual-constrained case
\end{itemize}

\subsection{Notation}

\begin{table}[h]
\centering
\caption{Notation Extension for Appendices.}
\begin{tabular}{cl|cl}
\toprule
$\xi_F$ & $4 + 4/\mathrm{gqa} + 6r$ & $\xi_W^{\mathrm{dec}}$ & $2 + 2/\mathrm{gqa} + 3r$ \\
$\xi_W^{\mathrm{all}}$ & $2 + 2/\mathrm{gqa} + 3r/\rho$ & $\tilde{D}$ & $\kappa_\rho \rho^{\alpha_\rho} d^{\beta_2-\beta_1} + \kappa_d$ \\
$\alpha_{\mathrm{attn}}$ & $2 + 2/\mathrm{gqa}$ & $\bar{S}$ & average context length \\
\midrule
$\bar{F}_p$ & $T_{\mathrm{lat}} \pi_H / (BS_{\mathrm{in}})$ & $\bar{M}_d$ & $T_{\mathrm{lat}} \beta_H / S_{\mathrm{out}}$ \\
$\eta$ & $\bar{M}_d / M_{\mathrm{budget}}$ & $\eta_p$ & $\bar{F}_p / M_{\mathrm{budget}}$ \\
$\Gamma$ & \multicolumn{3}{l}{$\xi_W^{\mathrm{dec}} d^2 b_w + 2\bar{S} d b_{kv}/\mathrm{gqa}$} \\
\bottomrule
\end{tabular}
\label{tab:ext_notation}
\end{table}

Here we present the notation~\autoref{tab:ext_notation} used in the Appendices, as an extension to~\autoref{tab:notation}.

\end{document}